\documentclass{article} 
\usepackage{iclr2027_conference,times}

\usepackage{amsmath,amsfonts,bm}

\def\eqref#1{equation~\ref{#1}}

\def\1{\bm{1}}

\DeclareMathAlphabet{\mathsfit}{\encodingdefault}{\sfdefault}{m}{sl}
\SetMathAlphabet{\mathsfit}{bold}{\encodingdefault}{\sfdefault}{bx}{n}

\usepackage{hyperref}
\usepackage{url}

\usepackage{graphicx}            
\usepackage{booktabs}            
\usepackage{multirow}            
\usepackage{amsfonts}       
\usepackage{nicefrac}       
\usepackage{microtype}      
\usepackage{xcolor}         
\usepackage[table]{xcolor}

\usepackage{multirow}
\usepackage{multicol}
\usepackage{makecell}
\usepackage{siunitx}
 \usepackage{mathtools}

\newcommand{\multhousand}[1]{%
  \num[round-mode=figures, round-precision=3,group-digits=false]{\fpeval{1000*(#1)}}%
}

\usepackage{algorithm}
\usepackage{amsmath}
\usepackage[noend]{algpseudocode}

\usepackage{wrapfig}
\usepackage{caption}
\usepackage{bbm}
\usepackage{booktabs}
\usepackage{tabularx}
\usepackage{array}
\usepackage{subcaption}
\usepackage{tablefootnote}

\usepackage{graphicx}

\usepackage{enumitem}
\usepackage{amsmath}
\usepackage{amsthm}

\newtheorem{Theorem}{Theorem}
\newtheorem{Proposition}{Proposition}
\newtheorem{Corollary}{Corollary}

\makeatletter
\newenvironment{algpanel}[1][]{%
  \par\addvspace{2pt}%
  \captionsetup{
    position=top,
    skip=0pt,
    singlelinecheck=off,
    justification=raggedright,
    labelfont=bf,
    labelsep=space
  }%
  \hrule height.8pt depth0pt
  \kern1pt
  \def\@algpanelcap{#1}%
  \ifx\@algpanelcap\@empty\else\caption*{#1}\fi
  \kern1pt
  \hrule height.4pt
  \kern3pt
}{%
  \kern3pt
  \hrule height.8pt depth0pt
  \par\addvspace{2pt}%
}
\makeatother

\title{BranchIP: Learning Adaptive Equivariant \linebreak Computation for Interatomic Potentials}

\author{%
  Laura Zichi\textsuperscript{*,1} \\
  \And
  Gil Harari\textsuperscript{*,1} \\
  \And
  Chuin Wei Tan\textsuperscript{1} \\
  \And
  Marc L. Descoteaux\textsuperscript{1} \\
  \And
  Albert Zhu\textsuperscript{1} \\
  \And
  Menghang Wang\textsuperscript{1} \\
  \And
  Yoel Zimmermann\textsuperscript{1} \\
  \And
  H.T. Kung\textsuperscript{1} \\
  \And
  Boris Kozinsky\textsuperscript{1,2}
  \AND
  \normalfont
  \textsuperscript{1}Harvard University \\
  \textsuperscript{2}Robert Bosch LLC Research and Technology Center \\
  \textsuperscript{*}Equal contribution.
}

\iclrfinalcopy 
\begin{document}

\maketitle

\begin{abstract}
Equivariant machine learning interatomic potentials (MLIPs) have revolutionized atomistic modeling, but accurate treatment of complex materials and molecular systems demands expensive models. This limits simulation length- and time-scales, with tensor products a key computational bottleneck. The recent emergence of foundation-scale MLIPs further exacerbates this challenge. We present Branch Interatomic Potential (BranchIP), a single-model framework for learned adaptive tensor product computation, trained with a novel distillation loss. In our experiments on two systems of physical interest, a heterogeneous catalysis system and a proton-conducting solid acid electrolyte, BranchIP accelerates MLIPs across model sizes by up to $2.4\times$  while reducing memory usage by up to $2.6\times$. This is achieved while maintaining physical fidelity. Furthermore, the learned adaptive computation provides model interpretability by revealing which interactions demand deeper computation and showing how computational depth relates to chemical complexity and dynamics.
\end{abstract}

\section{Introduction}

\begin{wrapfigure}{r}{0.38\textwidth}
    \vspace{-25pt}
    \centering
    \includegraphics[width=\linewidth]{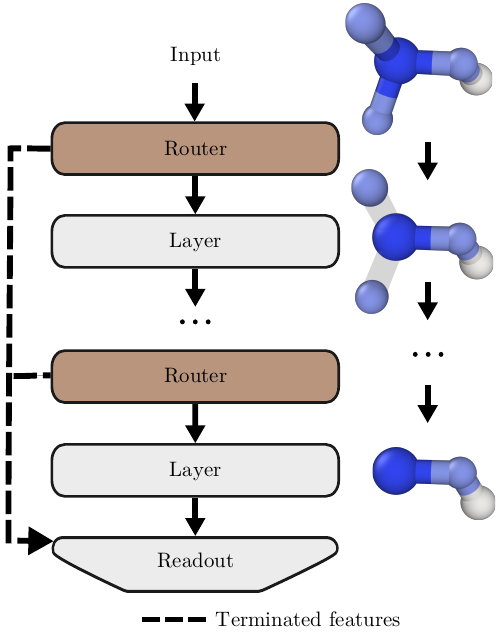}
    \caption{Overview of BranchIP. Routers placed before model layers terminate sufficiently refined features to reduce computation. Features represent local atomic interactions.}
    \label{fig:branchip_overview}
    \vspace{-15pt}
\end{wrapfigure}

Machine learning interatomic potentials (MLIPs) enable near first-principles modeling of diverse materials and molecules ranging from heterogeneous catalysts to biomolecular systems \citep{behler_generalized_2007, batzner_e3-equivariant_2022, batatia_mace_2022, owen_surface_2024, kozinsky_scaling_2023}. 
Practical applications, such as molecular dynamics (MD) simulations, require repeated evaluations of MLIPs, demanding both high predictive accuracy and fast inference~\citep{frenkel_understanding_1996}. 
Like vision and language models, MLIPs benefit from deep neural network (NN) architectures that progressively refine feature representations through successive nonlinear layers. 
The resulting deep MLIPs are largely considered black boxes and their interpretability remains an open challenge~\citep{kalita_machine_2025,zhang_roadmap_2025,creed_six_2026}.

The major breakthroughs in language and vision models from scaling dataset and model sizes \citep{kaplan_scaling_2020, hoffmann_training_2022, zhai_scaling_2022} recently motivated analogous efforts for MLIPs, resulting in pretrained foundation models (FMs) that generalize across diverse chemistry domains \citep{chen_universal_2022, park_scalable_2024, wood_uma_2026, batatia_foundation_2025-2, kavanagh_fast_2026}. However, computational cost hinders the practical use of large-scale MLIPs. The dominant bottleneck is the tensor product (TP), an operation which combines features from neighboring atoms to build the many-body geometric information essential for high accuracy. Most acceleration efforts reduce the cost of individual TPs through efficient formulations or optimized implementations \citep{passaro_reducing_2023, luo_enabling_2024, lee_flashtp_2025, bharadwaj_efficient_2025}. 
Recently, dynamic cutoffs have reduced the number of TPs by adapting the interaction graph \citep{han_smooth_2026}. 
This aligns with broader efforts in adaptive inference that combine potentials of different cost or fidelity \citep{birks_efficient_2026, wood_uma_2026, immel_adaptive-precision_2025}. 
These methods exploit the heterogeneous chemical complexity of realistic atomistic systems; only a subset of atomic environments requires deeper and more expensive architectures for accurate modeling. Existing approaches typically train and evaluate several independent models, resulting in expensive and complex training schemes with exorbitant parameter counts. Furthermore, assigning atomic environments to different models relies on hand-crafted descriptors that require prior knowledge of system dynamics---the very dynamics atomistic modeling often seeks to discover. 

To address these limitations, we introduce Branch Interatomic Potentials (BranchIPs), a single-model, learned adaptive computation framework for equivariant MLIPs inspired by early exits in deep NNs \citep{teerapittayanon_branchynet_2016} (Figure \ref{fig:branchip_overview}). 
As in classification, where early exits were first introduced, some inputs can be handled by shallower networks, while others require the full depth. 
Rather than applying the TP to every feature at every layer, BranchIP removes features from subsequent TPs once further refinement no longer improves accuracy.
The routing patterns reveal the TP refinement, usually representative of the body-order of model representations, required for various atomic environments. This improves the interpretability of the hidden model layers.

We implement BranchIP in the state-of-the-art equivariant MLIP Allegro~\citep{musaelian_learning_2023, tan_high-performance_2026}, termed Branchlegro. We demonstrate Branchlegro on two systems: a heterogeneous catalysis system comprising platinum surfaces with gaseous and adsorbed carbon monoxide (CO/Pt)~\citep{owen_surface_2024} and proton transport in the superprotonic solid acid electrolyte CsH$_2$PO$_4$ (CDP), a system without an obvious chemically motivated spatial partition~\citep{wang_revealing_2025}. For both systems, deriving chemical insight requires MD simulations over large length- and time-scales to sufficiently sample rare events, making computational efficiency essential. 

Our main contributions are:
\begin{enumerate}
    \item We develop a single-model framework for \textbf{feature-level adaptive TP} computation in MLIPs.
    \item We introduce a \textbf{fully learned dynamic router}, trained with a novel distillation objective, that smoothly routes features through TP computation obeying exact energy conservation.
    \item We implement the method in Allegro and prove it retains a twice continuously differentiable potential energy surface (PES) as well as equivariance.
    \item We demonstrate the method on bespoke and foundation MLIPs for CO/Pt and CDP systems, \textbf{accelerating inference and reducing memory usage} with limited degradation in accuracy.
    \item The learned routing provides model interpretability, revealing \textbf{which interactions require deeper equivariant computation} and how this varies with chemistry and dynamics.
\end{enumerate}

\section{Background and related work}
\textbf{MLIPs} use NNs trained on reference data, usually from expensive quantum methods, to approximate the PES. Given atomic positions and chemical species, $\{ \mathbf{r}_i, Z_i\}$, as input, the model predicts a total energy, $E$. Automatic differentiation then yields forces and potentially stresses.

\textbf{Allegro} \citep{musaelian_learning_2023} is a strictly-local, equivariant deep NN architecture for interatomic potentials. The model uses only operations that are equivariant under the three-dimensional Euclidean group, $E(3)$, consisting of rotations, translations, and reflections. As a result, the network predictions transform consistently with the underlying atomic geometry. Allegro decomposes the total energy into pair contributions from neighboring atoms, where neighbors are defined by a fixed cutoff distance. Tensor and scalar features represent each ordered atom pair, and an output multilayer perceptron (MLP) predicts the pair energy from the scalar features. At each layer, a TP combines features from neighboring pairs to incorporate many-body geometric information while preserving equivariance. 

\textbf{TP accelerations} largely focus on cheaper formulations \citep{xu_cartesian-3j_2026, passaro_reducing_2023, luo_enabling_2024} and optimized kernels \citep{lee_flashtp_2025, tan_high-performance_2026, bharadwaj_efficient_2025} and are complementary to BranchIP.

\textbf{Adaptive interatomic potential architectures} seek to reduce inference cost while maintaining accuracy. \textbf{ML-MIX} relies on a user-defined assignment of expensive and inexpensive interatomic potentials to predefined spatial regions \citep{birks_efficient_2026}. Blending forces from different models results in non-conservative dynamics. \textbf{Adaptive-precision potentials} use a continuous switching parameter, computed from hand-crafted descriptors, to mix the energies of fast and accurate potentials \citep{immel_adaptive-precision_2025}. The \textbf{Universal Models for Atoms} architecture uses a mixture-of-linear-experts framework to scale FMs without increasing inference cost \citep{wood_uma_2026}. A router predicts expert weights from global time-invariant system properties, producing a single, fixed routing assignment during simulations. \cite{liu_scaling_2026} modified this to element-wise routing. \textbf{Static multifidelity Allegro} assigns models to fixed, predefined regions based on chemical intuition \citep{nascimento_mixture_2026}. \citet{han_smooth_2026} introduce a \textbf{dynamic cutoff} for atoms using a fixed function. Because BranchIP adapts computation within a single model, it is complementary to these architectures.

\textbf{Early exits}, introduced by~\citet{teerapittayanon_branchynet_2016}, adapt NN inference to input complexity by attaching intermediate prediction heads. Sufficiently simple inputs can terminate before the final layer, reducing average compute while retaining full depth for harder inputs.
Early exits were initially developed for convolutional NNs ~\citep{teerapittayanon_branchynet_2016, huang_multi-scale_2018} 
and later extended to transformers~\citep{xin_deebert_2020}
and graph NNs (GNNs)~\citep{di_francesco_early-exit_2026}. 
\textbf{Learned adaptive computation} has also been explored at finer levels of granularity, including 
dynamically skipping network blocks~\citep{wang_skipnet_2018, wu_blockdrop_2018}, 
selectively evaluating feature channels~\citep{gao_dynamic_2019, hua_channel_2019}, 
and varying message-passing computation across nodes in GNNs~\citep{han_turning_2024, abbahaddou_admp-gnn_2025}. 
These approaches show that computation can also be allocated dynamically across intermediate components. 
To the best of our knowledge, learned dynamic allocation of equivariant computation across model depth has not been explored. 

\section{Learning adaptive equivariant computation}

\begin{figure}[t]
  \centering
  \begin{subfigure}[t]{0.33\textwidth}
    \vspace{0pt}
    \centering
    \caption{Branchlegro schematic.}
    \label{fig:branchlegro_schematic}
    \includegraphics[width=\linewidth]{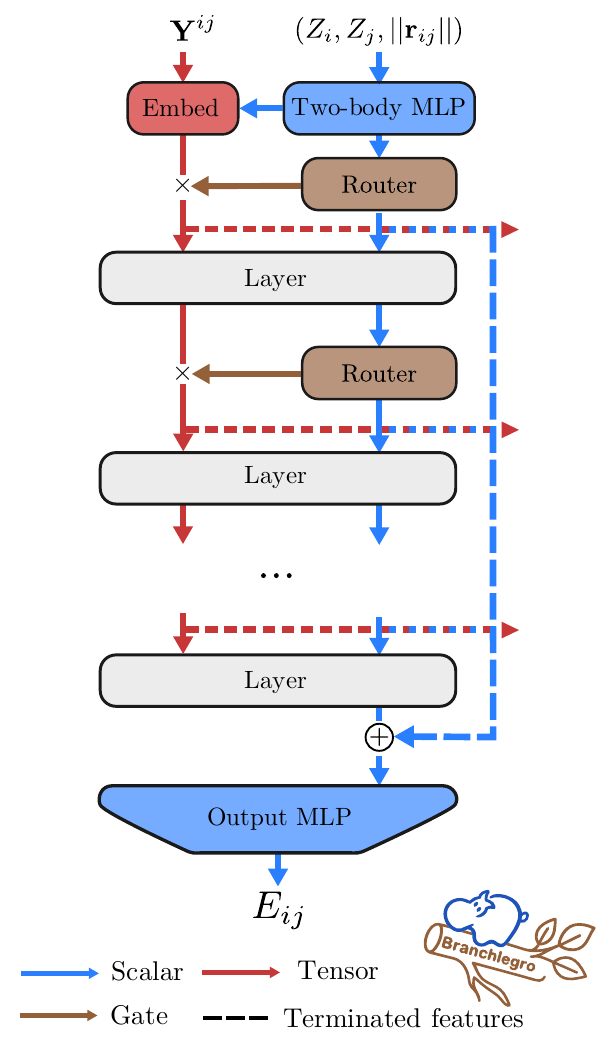}
  \end{subfigure}
  \hfill
  \begin{subfigure}[t]{0.6\textwidth}
    \vspace{0pt}
    \caption{Pseudocode of the Branchlegro inference.}
    \small
    \begin{algpanel}[\textsc{Branchlegro forward pass}]
    \label{alg:branchlegro}
    \begin{algorithmic}[1]
    \Require Edge set $\mathcal{E}$; embedding $(\mathbf{x}^{ij,0}, w^{ij,0}, \bar{\mathbf{V}}^{ij,0}, \mathbf{Y}^{ij})$.
    \State $\mathcal{E}_0 \gets \mathcal{E}$ 
    \For{$l = 1,\dots,L$}
        \If{$l-1$ has a router}
            \State $f^{ij,l-1}\gets
            \mathcal{F}\!\left(\mathrm{MLP}_{\mathrm{router}_{l-1}}(\mathbf{x}^{ij,l-1})\right),\; ij\in\mathcal{E}_{l-1}$  
            \State $\mathbf{V}^{ij,l-1}\gets
            f^{ij,l-1}\bar{\mathbf{V}}^{ij,l-1},\; ij\in\mathcal{E}_{l-1}$  
            \State $\mathcal{E}_{l}\gets
                \{ij\in\mathcal{E}_{l-1}:f^{ij,l-1}>0\}$ 
        \Else
            \State $\mathbf{V}^{ij,l-1}
            \gets
            \bar{\mathbf{V}}^{ij,l-1},\; ij\in\mathcal{E}_{l-1}$
            \State $\mathcal{E}_{l}\gets\mathcal{E}_{l-1}$
        \EndIf
        \State $\bar{\mathbf{V}}^{ij,l}  
        \gets
        \left(
        \sum_{k:(i,k)\in\mathcal{E}}
        w^{ik,l-1}\mathbf{Y}^{ik}
        \right)
        \otimes
        \mathbf{V}^{ij,l-1},\; ij\in\mathcal{E}_{l}$ 
        \State $
        \mathbf{z}^{ij,l}
        \gets
        \mathbbm{1}\{ij\in\mathcal{E}_{l}\}
        \left(\bar{\mathbf{V}}^{ij,l}\right)_{\ell=0},
        \; ij\in\mathcal{E}$ 
        \State $(\mathbf{x}^{ij,l}, w^{ij,l})
        \gets
        \mathrm{MLP}_{\mathrm{latent}_l}
        \left(
        \left[
        \mathbf{x}^{ij,0},
        \dots,
        \mathbf{x}^{ij,l-1},
        \mathbf{z}^{ij,l}
        \right]
        \right)$ 
    \EndFor
    \State \Return $E \leftarrow \sum_{ij\in\mathcal{E}}\mathrm{MLP}_{\mathrm{out}}
    \left[
    \mathbf{x}^{ij,0},
    \mathbf{x}^{ij,1},
    \dots,
    \mathbf{x}^{ij,L}
    \right]$ 
    
    \end{algorithmic}
    \end{algpanel}
  \end{subfigure}
  \caption{(a) Schematic of the Branchlegro network. Blue and red arrows represent scalar and tensor features. $Z_i$ is the atomic species of atom $i$, $\mathbf{r}_{ij}$ is the relative displacement vector from atom $i$ to $j$. For simplicity, the residual update is not presented. Continuation gates modulate the tensor features and they are routed out in inference when they have a zero-gate (dashed red arrows). Routed scalar features are concatenated back for the energy readout (dashed blue arrows). (b) Pseudocode of the Branchlegro forward pass. For edge $ij$ and layer $l$, $\mathbf{x}^{ij,l}$ are the scalar features, $w^{ij,l}$ are the environment weights, $\bar{\mathbf{V}}^{ij,l}$ are the tensor features, $\mathbf{Y}^{ij}$ are the spherical harmonic embeddings, $L$ is the number of layers, and $\mathcal{F}$ is a gating function. The initial embedding is the standard Allegro one.}
  \label{fig:branchlegro}
\end{figure}

BranchIP augments an MLIP with lightweight routers placed before selected TP layers. 
At layer $l$, for each feature $k$ (node or edge, depending on the architecture), a router produces an invariant continuation gate, $f^{k,l} \in[0,1]$, that modulates the feature's downstream contribution.
The router may be learned or heuristic-based and may use any local information, including intermediate embeddings.
At inference, zero-gated features are omitted from subsequent TPs, reducing compute and memory usage; we refer to this as sparse evaluation. 
During training, zero-gated features are retained to preserve supervision for downstream routers; we refer to this as dense evaluation.
Therefore, BranchIP allocates depth adaptively across features, rather than applying the full depth uniformly.

Because conservative MLIPs obtain forces from energy derivatives, dynamic routing must remain part of the differentiable network computation. 
BranchIP therefore multiplies each feature by its continuation gate before the next TP. The gate makes routing a continuous operation on the features rather than a discrete selection, such that a feature's contribution decays smoothly to zero as it is routed out.
To preserve the regularity of the PES, the continuation gate must be continuously differentiable to the same order as the MLIP it augments, with a flat exact-zero region. The flat region ensures that a feature reaches zero smoothly enough that its derivatives also vanish; as a result it carries no residual contribution when it is terminated. Zero-gated features can therefore be omitted from subsequent TP computation without changing the energy or its derivatives.

This design ensures that the sparse inference path reproduces the dense evaluation exactly, such that the routed model inherits the smoothness and symmetry guarantees of the architecture it augments.
\begin{Theorem}
\label{theo:main}
Let an $E(3)$-equivariant twice continuously differentiable ($C^2$) MLIP be augmented with routers such that:
(i) the continuation gate is $E(3)$-invariant, $C^2$ with a flat exact-zero region, and the same gating function is applied to every feature.; (ii) features are removed from TP computation only where the gate is exactly zero.
Then, for all atomic coordinates $\mathbf{R}$:
\begin{enumerate}[leftmargin=1cm]
  \item \emph{Regularity.} The PES is $C^2$ in $\mathbf{R}$, and the sparse computation yields its exact first and second derivatives. In particular, forces and stresses obtained from the sparse computation are exact.
  \item \emph{Symmetry.} $E(3)$-equivariance of the tensor features, invariance of the energy, and permutation symmetry are preserved.
\end{enumerate}
\end{Theorem}


\emph{Proof intuition.} 
The argument rests on the equivalence between the dense and sparse evaluations. 
The TP vanishes when its gated input is zero. Thus, by (ii), removed features do not contribute to the TP and both evaluations produce the same energy.
For derivatives (e.g., forces), the sparse evaluation never builds backward computation branches for removed features. 
Thus, any nonzero derivative of the gate at a removed feature would contribute to the dense evaluation but be absent from the sparse evaluation. 
To prevent this, (i) requires the gate to be identically zero on an interval before a feature is removed, and to be $C^2$. Its derivatives therefore vanish on the interval and, by continuity, at the boundary, so each removed feature has zero value, zero slope, and zero curvature.
Removing zero-gated features leaves the energy and its derivatives unchanged, so the sparse path computes the exact gradient of the same $C^2$ energy surface as the dense path.
Finally, by the invariance and identical application in (i), the active feature set is unchanged under rotation and is permuted consistently under relabeling, so routing preserves the symmetries of the model.
The same argument holds at any order: a $C^m$ gate yields a $C^m$ surface when the underlying MLIP is $C^m$.

For learned routing, we propose lightweight MLP routers trained by distilling a frozen full-compute teacher MLIP, from which the student BranchIP model is initialized. 
In the MLPs we use a sigmoid activation function and interpret the output as continuation probabilities, $\nu^{k,l}\in(0,1)$.
As a sigmoid cannot reach exact zeros, we define the gate as $f^{k,l}=\mathcal{F}(\nu^{k,l})$, where $\mathcal{F}$ satisfies the properties above. 
To train the router MLPs, we propose the following distillation loss, inspired by binary cross-entropy. 
We define a student-teacher prediction discrepancy, $\delta^k = \left\lVert \mathbf{y}^k_{\mathrm{teacher}} - \mathbf{y}^k_{\mathrm{student}} \right\rVert_2$ on a chosen output $\mathbf{y}^k$ (e.g., energy, forces). A low $\delta^k$ indicates that further TP computation is unnecessary, whereas a high $\delta^k$ indicates that computation should continue. User-defined stop and continue tolerances, $\tau^l_{\mathrm{cont}}$ and $\tau^l_{\mathrm{stop}}$, convert $\delta^k$ into learnable router targets optimized with the distillation loss,
\begin{equation}\label{eq:loss_distill}
  \mathcal{L}_{\mathrm{distill}}
=
  -\sum_{l=1}^{L}\sum_{k\in\mathcal{K}_l}
  r^{k,l}
  \left[
  \mathbbm{1}_{\{\delta^{k}>\tau_{\mathrm{stop}}^{l}\}}
  \log \nu^{k,l}
  +
  \mathbbm{1}_{\{\delta^{k}\leq\tau_{\mathrm{cont}}^{l}\}}
  \log\!\left(1-\nu^{k,l}\right)
  \right],
  \end{equation}

where $\mathcal{K}_l$ is the set of active features in layer $l$ and $r^{k,l}= \mathrm{sg} [ \prod_{m<l}\nu^{k,m}]$ weights each term by the fraction of feature $k$'s signal that reaches router $l$.
$\mathrm{sg}$ denotes the stop-gradients operator. This prevents earlier routers from reducing downstream distillation loss by decreasing their continuation probabilities. The loss penalizes unnecessary continuation when the discrepancy falls below the continue tolerance and premature stopping when it exceeds the stop tolerance. The tolerances should obey $\tau_{\mathrm{cont}}\leq\tau_{\mathrm{stop}}$ and discrepancies between the two tolerances are left unsupervised, providing a transition region between continuation and stopping.

\textbf{Branchlegro.} 
We implement BranchIP in Allegro, denoted Branchlegro. 
In the implementation, MLPs route features that represent edges between atoms. 
We train the routers using the proposed distillation loss with force predictions (see Appendix~\ref{app:branchlegro} for formulation and gating function).
Only the TP is sparsified, with terminated edges continuing to receive scalar latent updates. The skipped TP outputs are zero-padded to the full edge set and environment weights are evaluated over all edges at every layer. 
Allegro's radial envelope vanishes to second order at the cutoff, so the architecture is $C^2$. We therefore use a $C^2$ gate, identically zero below a threshold, matching that regularity. A smoother gate is available by increasing $p$ in the continuation function.
The Branchlegro architecture schematic and pseudocode are presented in Figure~\ref{fig:branchlegro}. 
Branchlegro therefore satisfies (i) and (ii), and Theorem~\ref{theo:main} follows from Propositions~\ref{prop:energy}--\ref{prop:symmetry} and Corollary~\ref{cor:regularity}, proved in Appendix~\ref{app:theory}.

\section{Results}

\begin{figure}[tb]
    \centering
    \includegraphics[width=\textwidth]{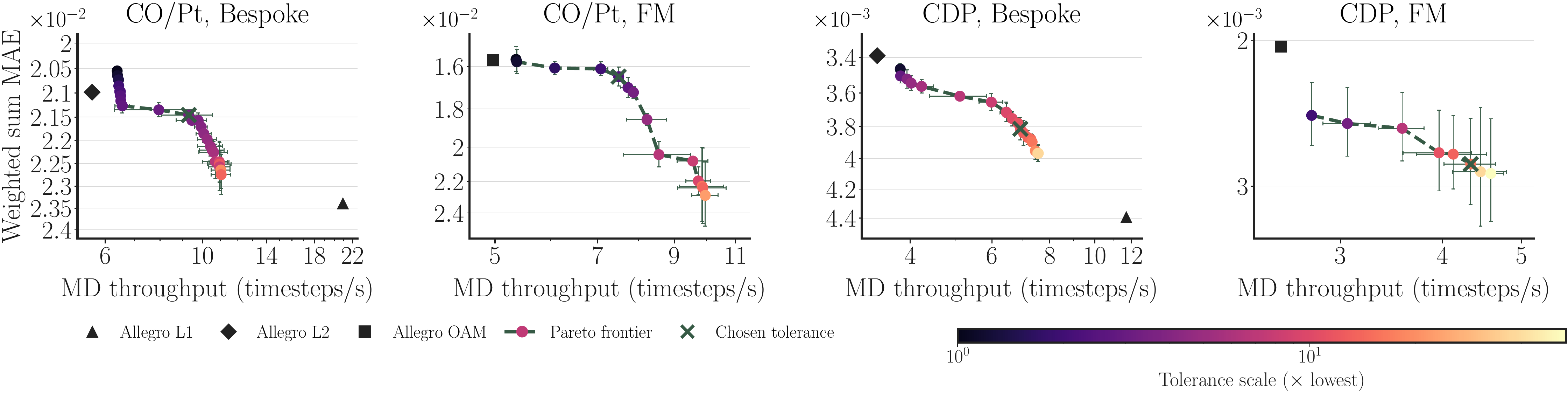}
    \caption{Accuracy--throughput trade-offs for Branchlegro on CO/Pt and CDP using bespoke models and FMs. 
    Each point reports the mean test weighted-sum MAE and MD throughput over seven seeds; error bars denote standard deviation.
    The inverted $y$-axis indicates higher accuracy (higher is better). 
    Weights for CO/Pt are 1:1:1 for per-atom energy, forces, and stress; for CDP are 10:1 for per-atom energy and forces. 
    Throughput was measured as described in Appendix~\ref{app:md_throughput} using system sizes of 4,160 and 512 atoms for CO/Pt bespoke models and FMs, and 14,336 and 768 atoms for CDP bespoke models and FMs. 
    Black markers denote Allegro baselines. Only Pareto frontier points are shown; tolerance scaling stops once no new frontier points are found.
    The $\times$ markers indicate the configurations used for the main results.
    }
    \label{fig:pareto}
\end{figure}

We trained Branchlegro models from bespoke Allegro potentials and fine-tuned Allegro-OAM-L FMs~\citep{kavanagh_fast_2026} by optimizing a joint objective comprising the proposed distillation loss and standard energy, force, and, for CO/Pt, stress supervision using seven random seeds (Appendix~\ref{app:setup}).
We evaluate each Branchlegro model against the Allegro teacher it was distilled from: two-layer (L2) models for bespoke potentials and the fine-tuned Allegro-OAM-L for the FMs. For the bespoke models, we also include one-layer (L1) Allegro as an efficiency comparison at lower accuracy, whereas no lower-compute variant is available for Allegro-OAM-L.

We examine the trade-off between predictive accuracy and MD throughput in Branchlegro by increasing the routing tolerances (Figure~\ref{fig:pareto}). 
Across all four settings, relaxing the tolerances yields substantial throughput gains with only moderate increases in error up to a model-dependent knee. Some tolerance profiles fall near transitions between routing regimes, where small differences in optimization can change which routers are favored, increasing the variance in throughput across seeds. For the CDP FM models, two of seven seeds result in worse validation and test errors across all tolerance profiles leading to the larger error variance. We select models near the corresponding validation-set knees for the system-specific scaling and physical studies that follow. These serve as representative models; in practice, Branchlegro routing tolerances can be tuned to match the accuracy and throughput requirements of a given application.

\begin{figure}[t]
    \centering
    \includegraphics[width=\textwidth]{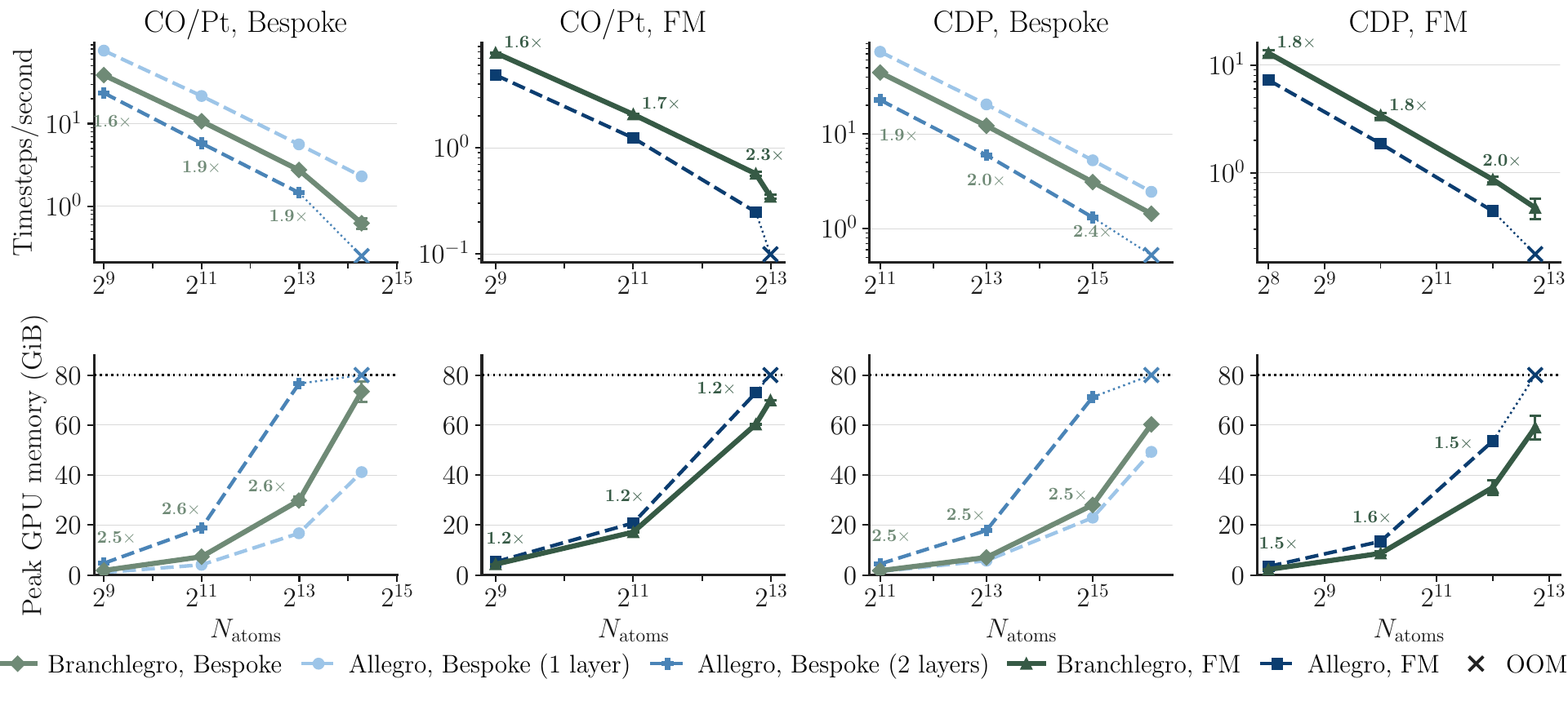}
    \caption{MD throughput and memory usage scaling of Allegro and Branchlegro on CO/Pt and CDP. 
    Branchlegro values are averaged over seven seeds, with error bars indicating the standard deviation.
    All evaluations followed the setup in Appendix~\ref{app:md_throughput}. 
    Speedups and memory reductions are relative to the corresponding full-sized Allegro model. Crosses denote out-of-memory (OOM).
    }
    \label{fig:throughput}
\end{figure}

\newcommand{\exitlayer}[1]{%
  \prescript{\scalebox{0.5}{$\hookleftarrow$}}{}{#1}%
}
\begin{table*}[b]
\centering
\caption{
Accuracy (MAE) of Allegro and Branchlegro on test datasets. Branchlegro errors are shown as mean $\pm$ standard deviation over seven seeds. Superscript $\scalebox{0.5}{$\hookleftarrow$}$ denotes a router before a layer.
}
\label{tab:accuracy_speed}

\setlength{\tabcolsep}{4pt}
\renewcommand{\arraystretch}{0.95}
\small

\begin{tabular}{@{}lccc@{\qquad}cc@{}}
\toprule
& \multicolumn{3}{c}{\textsc{CO/Pt}}
& \multicolumn{2}{c}{\textsc{CDP}} \\
\cmidrule(lr){2-4}
\cmidrule(lr){5-6}

\textsc{Architecture}$_{\{\textsc{Layers}\}}$
& $E$ {\scriptsize [meV/atom]}
& $F$ {\scriptsize [meV/\AA]}
& $S$ {\scriptsize [meV/\AA$^3$]}
& $E$ {\scriptsize [meV/atom]}
& $F$ {\scriptsize [meV/\AA]} \\
\midrule

\textsc{Allegro}$_{\boldsymbol{\{0\}}}$
&
$\multhousand{0.00250175547}$
&
$\multhousand{0.06735978317}$
&
$\multhousand{0.0002975983867}$
&
$\multhousand{0.000805370123}$
&
$\multhousand{0.04026606266}$
\\

\textsc{Allegro}$_{\boldsymbol{\{0,1\}}}$
&
$\multhousand{0.00253073388}$
&
$\multhousand{0.06015199646}$
&
$\multhousand{0.0002666349784}$
&
$\multhousand{0.0005795040441}$
&
$\multhousand{0.03150236241}$
\\

\textsc{Branchlegro}$_{
\boldsymbol{\{\exitlayer{0},\exitlayer{1}\}}
}$
&
$\multhousand{0.002159170346}$
${\scriptstyle \pm \multhousand{0.0001042370563}}$
&
$\multhousand{0.06193266317}$
${\scriptstyle \pm \multhousand{0.0002161352951}}$
&
$\multhousand{0.0002765569909}$
${\scriptstyle \pm \multhousand{5.536132909e-06}}$
&
$\multhousand{0.0006284514387}$
${\scriptstyle \pm \multhousand{2.604456412e-05}}$
&
$\multhousand{0.03565826833}$
${\scriptstyle \pm \multhousand{0.0006819899932}}$
\\

\cmidrule(lr){1-6}

\textsc{Allegro}$_{\boldsymbol{\{0,1,2,3\}}}$
&
$\multhousand{0.001557818123}$
&
$\multhousand{0.04535222692}$
&
$\multhousand{0.0002335597348}$
&
$\multhousand{0.0003391213118}$
&
$\multhousand{0.01899387076}$
\\

\textsc{Branchlegro}$_{
\boldsymbol{\{
\exitlayer{0},\exitlayer{1},\exitlayer{2},\exitlayer{3}
\}}
}$
&
$\multhousand{0.001766895348}$
${\scriptstyle \pm \multhousand{0.0002981581678}}$
&
$\multhousand{0.04734586662}$
${\scriptstyle \pm \multhousand{0.00121551682}}$
&
$\multhousand{0.0002658904577}$
${\scriptstyle \pm \multhousand{7.910900502e-06}}$
&
$\multhousand{0.00044146041}$
${\scriptstyle \pm \multhousand{5.709050628e-05}}$
&
$\multhousand{0.02661726253}$
${\scriptstyle \pm \multhousand{0.003172767509}}$
\\

\bottomrule
\end{tabular}
\end{table*}

\textbf{Speedup and memory.}
Figure~\ref{fig:throughput} shows the MD throughput speedup and memory usage of Branchlegro over Allegro (Appendix~\ref{app:md_throughput}). 
In our experiments, Branchlegro models achieve $1.6$--$2.4 \times$ speedup across model sizes and systems. 
Notably, accelerations emerge even at small scale: the ${\sim}100$K-parameter bespoke Branchlegro CDP model is $1.9\times$ faster with 2,048~atoms. 
Branchlegro also reduces memory requirements, enabling inference regimes inaccessible to Allegro by reducing memory usage by $1.2$--$2.6 \times$. 
The FMs may achieve smaller memory reductions due to their significantly larger hidden dimension and latent MLPs.
The bespoke models achieve significant memory reductions even in the regime of smaller speedups (e.g., $2.5\times$ memory reduction for CO/Pt bespoke on 512~atoms, with only a $1.6\times$ speedup).
The speedup and memory reductions reported reflect the selected models and systems and not an inherent upper limit of Branchlegro performance.
For example, by choosing a CO/Pt FM Branchlegro model with a larger tolerance scale, the speedup increases to $3.2\times$ and memory reduction to $1.8\times$, at the cost of further accuracy degradation (Figure~\ref{fig:pareto}).  

\textbf{Accuracy.} We evaluate per-atom energy, force, and stress mean absolute errors (MAEs) (Table~\ref{tab:accuracy_speed}), alongside physical fidelity to ground-truth calculations (Appendix~\ref{app:physical_fidelity}). Branchlegro maintains comparable MAEs 
while retaining energy profiles and structural and dynamical observables (Appendix~\ref{app:physical_fidelity}).
Notably, bespoke Branchlegro models achieve this accuracy despite most edges stopping before the first TP layer.
For CO/Pt, the L1 Allegro model exhibits a lower degradation in errors relative to the L2 model than in CDP. This suggests that CO/Pt requires fewer interaction layers, consistent with the lower errors achieved by Branchlegro. Branchlegro models obey energy conservation in MD (Appendix~\ref{app:energy_conservation}). We further demonstrate the practicality of Branchlegro models in realistic MLIP applications, large length- and time-scale MD. The CDP MLIP provides explicit simulation of long-range proton transport from which observables such as proton-diffusion activation energy can be determined. Branchlegro yields a proton-diffusion activation energy of $E_a=0.40$~eV, within the experimental range of $0.39$--$0.43$~eV~\citep{haile_solid_2007, ishikawa_proton_2008} and close to its L2 teacher, whereas L1 Allegro is outside the experimental range with $E_a=0.46$~eV.

\begin{figure}[p]
  \centering
  \includegraphics[width=1.0\textwidth]{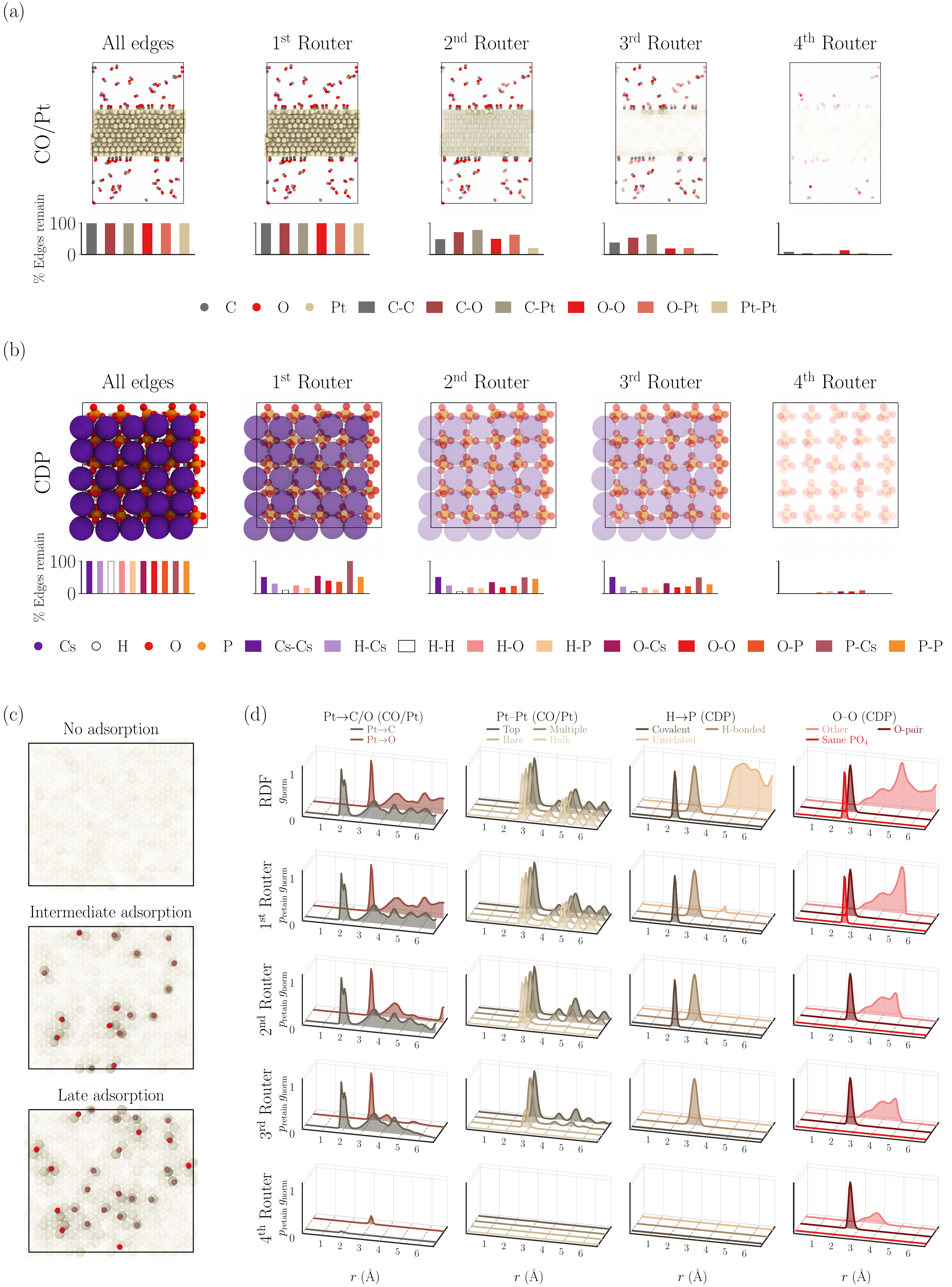}
  \caption{Adaptive routing in CO/Pt and CDP Branchlegro FMs. Atom opacity represents the cumulative gating of its connected edges retained in the TP operations after each router.
  (a) and (b) CO/Pt and CDP structures before routing and after Routers 1–4. For clarity, the CO/Pt visualization retains only the two closest planes of Pt atoms. 
  (c) Top views of CO/Pt at Router 3 for configurations with varying CO adsorption. Only the upper two Pt layers and CO molecules within 2.2\,\text{\AA} of the uppermost Pt layer are shown.
  (d) Unweighted normalized radial distribution functions (RDFs), $g_\text{norm}$, and retention-weighted RDFs, $p_{\mathrm{retain}}g_\text{norm}$, after Routers 1–4 for Pt--C/O and Pt--Pt interactions in CO/Pt and H--P and O--O in CDP. RDFs are normalized once using the initial peak value. 
  }
  \label{fig:routing_vis}
\end{figure}

\subsection{Interpretability}

The learned routing provides model interpretability by showing which atomic environments require deeper equivariant computation. We analyze the routing patterns of the two systems and focus on the four routers in the Branchlegro FMs, as they provide a richer analysis than the two routers of the bespoke models.

\textbf{CO/Pt.} 
The system consists of a Pt slab interacting with CO gas molecules, where carbon atoms adsorb to the surface.
The system contains a relatively regular Pt bulk together with a chemically heterogeneous and complex surface. 
On the surface, Pt--C adsorption, C--O bonding, lateral CO interactions, and local surface Pt geometry are strongly coupled~\citep{owen_surface_2024}.
The routers retain mostly edges involving carbon and oxygen and a small subset of the Pt--Pt edges (Figure~\ref{fig:routing_vis}a). 
Notably, this routing tracks the chemistry of the reactive surface: the fourth router retains adsorbed CO while routing out most gas-phase CO, and the second and third routers preferentially retain edges of Pt atoms with adsorbed CO relative to those without (Figure~\ref{fig:routing_vis}c). We analyze these patterns as a function of edge distance in Figure~\ref{fig:routing_vis}d for Pt--C/O edges (centered on Pt) and for Pt--Pt edges.

Pt--C and Pt--O edges are retained in the final layer at around $2$\,\AA\ and $3$\,\AA, the expected distances of adsorbed C~\citep{ogletree_leed_1986} and its O atom. In contrast, Pt edges with gaseous CO are terminated by the final router. For Pt--Pt edges, we classify the central Pt atom as bulk, bare surface, or adsorbate-bound in a linear single coordination (top) or multiple coordination (multi) configuration (see Appendix~\ref{app:interpretability}).
Strikingly, the second router differentiates bulk, bare surface, and adsorbate-bound Pt atoms, while the third router further separates the two adsorbate configurations.
Since Pt--Pt separations are largely unchanged by adsorption, this pattern indicates that the routers identify these environments as more complex, requiring higher body-order information, and allocate additional TP refinement to them accordingly.

\textbf{CDP.} 
The system is composed of a relatively localized Cs and P framework, orientationally disordered PO$_4$ tetrahedra, and highly mobile protons (H).
The framework edges (Cs--Cs, Cs--P, and P--P) are mainly retained by the earlier routers, but are entirely terminated by the final router (Figure~\ref{fig:routing_vis}b). 
Conversely, a substantial portion of edges involving H or O atoms are terminated in early routers, while a small subset is retained through the final router. This suggests these interactions require deeper many-body computation.
We inspect the routing patterns over edge distance in Figure~\ref{fig:routing_vis}d for H--P edges (centered on H) and for O--O edges.

For H--P, the edges are classified by the proton's instantaneous bonding relationship to the PO$_4$ group: covalent, when the proton is covalently bonded to one of its oxygens; hydrogen-bonded, when the proton forms a hydrogen bond to one of the oxygens in the same PO$_4$ group (H$\cdots$O--P); and unrelated, when neither holds. The classes are well separated in distance, with the covalent and hydrogen-bonded classes peaking near $2.3$ and $3.0$\,\AA\ and the unrelated class beginning beyond $4$\,\AA. As the proton rattles on sub-picosecond timescales between the two oxygens of a hydrogen bond~\citep{wang_revealing_2025}, which belong to different PO$_4$ groups, a given phosphate pair exchanges the two labels repeatedly along a trajectory.
In the Branchlegro model, the first router terminates nearly all unrelated H--P edges, the third terminates the covalent edges, and the fourth terminates the hydrogen-bonded edges. The early removal of unrelated edges indicates that distant phosphate environments contribute little once the local proton environment has been established. The two bonded classes differ in how much of the H--P geometry the covalent bonding already fixes. For the covalent class, the near-rigid P--O--H geometry fixes it~\citep{wang_revealing_2025}, accounting for the first peak. The H--P edge therefore adds little beyond the covalent H--O edge already present in the same proton-centered environment. For the hydrogen-bonded class, no covalent bond constrains the H--P geometry. The geometry instead depends on how two reorienting PO$_4$ groups are arranged relative to one another, and thus responds to the rotation of either. These edges therefore remain for deeper many-body interactions.

For O--O, the edges are classified by the bonding relationship between the two oxygens: same-PO$_4$, for an O--O pair within one group; O-pair, for an O--H$\cdots$O hydrogen bond between two groups that share a proton~\citep{wang_revealing_2025}; and other, for the remaining pairs in different groups. 
The same-PO$_4$ edges terminate at the second router, while the O-pair edges are retained at full amplitude through all four routers. Because the two classes nearly coincide in distance, this separation cannot come from radial distance alone. 
At two-body order (Router 1), the router sees only species and distance, and accordingly retains both classes. The first TP adds, at each oxygen, the angle between the edge and that oxygen's own P--O bond, which identifies the pair as intramolecular. The edge is then redundant, since the P--O edges already encode the tetrahedron's orientation, and the second router terminates the edge. 
In contrast, an O-pair edge depends on the position of the proton between the two oxygens. That position is continually modified by rattling within the O-pair and by O-sharing, in which two protons from different O-pairs simultaneously share a single oxygen~\citep{wang_revealing_2025}. Resolving this O-pair edge requires a larger neighborhood than in the intramolecular case, and therefore features of higher many-body order, consistent with the model retaining these edges through the final router.

\section{Conclusion}
We introduce BranchIP, a learned, dynamic adaptive inference framework for equivariant MLIPs that obeys energy conservation. The framework leverages the insight that features representing certain atomic environments have sufficient TP refinement and can avoid additional computation in later layers. Through our development of a novel distillation loss, learned routers provide model interpretability and accelerate inference without relying on chemical intuition. Our implementation, Branchlegro, in the MLIP Allegro achieves significant inference speedups and memory efficiency on complex reactive systems: heterogeneous CO/Pt and CDP, a system with no clear spatial partition. Our results show that adaptive computation creates faster MLIPs without hand-crafted routing, exorbitant parameter counts, or loss of physical fidelity.

\textbf{Limitations.}
The Branchlegro implementation only reduces the TP computations, while the rest of the computations done in the layer are retained as in a standard Allegro model. This limits the speedup potential of Branchlegro, especially when there are other expensive operations (e.g., large edge-wise MLPs). In addition, the learned routing scheme with the novel distillation loss requires a pretrained teacher model, increasing the cost of the overall training. The experiments included in this work focus on two systems.

\section{Acknowledgments and Disclosure of Funding}
We thank Anders Johansson and Itay Lavie for helpful discussions on LAMMPS usage and the manuscript preparation. 
A portion of the computations for this work were performed on the FASRC Cannon cluster supported by the FAS Division of Science Research Computing Group at Harvard University.
An award of computer time was provided by the INCITE program. This research used resources of both the Argonne and Oak Ridge Leadership Computing Facilities, which are DOE Office of Science User Facilities supported under contracts DE-AC02-06CH11357 and DE-AC05-00OR22725.
L.Z. was supported by the National Science Foundation Graduate Research Fellowship under Grant No. DGE-2140743
and Grant No. DGE-2630152. 
A.Z.~is supported by the U.S.~Department of Energy, Office of Science, Office of Advanced Scientific Computing Research, Department of Energy Computational Science Graduate Fellowship under Award Number DE-SC0025528.
This work was supported by the National Science Foundation, Office of Advanced Cyberinfrastructure (OAC), under Award No. 2118201 and the Harvard University Materials Research Science and Engineering Center Grant No. DMR-2011754.

\bibliographystyle{iclr2027_conference}
\bibliography{references}

\appendix

\section{Branchlegro details} \label{app:branchlegro}

\subsection{Force Distillation Loss Formulation}
We detail the exact formulation for the force distillation loss component used to train our Branchlegro models. The features in Allegro are edge-based, so $\mathcal{K}=\mathcal{E}$. For the force discrepancy, we use the $\ell_2$ norm of the force difference between the teacher and student predictions:
\begin{equation}
    \delta^F_i = \left\lVert \mathbf{F}_i^{\mathrm{teacher}} - \mathbf{F}_i^{\mathrm{student}} \right\rVert_2,
\end{equation}
where $\mathbf{F}_i^{\mathrm{teacher/student}}$ is the force prediction of the teacher/student on atom $i$.
As forces are node-resolved but the routers predict edge-wise probabilities, we convert the atomic force discrepancies to edge discrepancies:
\begin{equation}
    \delta^{F}_{ij} = \frac{\delta_i^F+\delta_j^F}{2}.
\end{equation}

Thus, for Branchlegro's edge features, Equation~\ref{eq:loss_distill} becomes:
\begin{equation}
  \mathcal{L}_{\mathrm{distill}}
=
  -\sum_{l=1}^{L}\sum_{ij\in\mathcal{E}_l}
  r^{ij,l}
  \left[
  \mathbbm{1}_{\{\delta^{ij}>\tau_{\mathrm{stop}}^{l}\}}
  \log \nu^{ij,l}
  +
  \mathbbm{1}_{\{\delta^{ij}\leq\tau_{\mathrm{cont}}^{l}\}}
  \log\!\left(1-\nu^{ij,l}\right)
  \right],
  \end{equation}
where 
$\nu^{ij,l} = \mathrm{MLP}_{\mathrm{router}_{l}}(\mathbf{x}^{ij,l})$ and
$r^{ij,l}= \mathrm{sg} [ \prod_{m<l}\nu^{ij,m}]$. 

\subsection{Continuation-gating function}
For the smooth continuation-gating function, we use 
\begin{equation}
   \mathcal{F}(x) = 1-\exp\!\left[
  -\left(
  \left(\frac{x-T}{w}\right)_+
  \right)^p
  \right]
\end{equation}
where $(a)_+ \coloneqq \max\{a,0\}$. For $p\in\mathbb{N}$ and $w>0$, this construction ensures that $\mathcal{F}(x)\in[0,1)$, $\mathcal{F}\in C^{p-1}$, and $\mathcal{F}(x)=0$ for all $x \leq T$. Thus, we choose $p=3$ to ensure the resulting PES is $C^2$. 
In our implementation, we use $T=0.3$ and $w=0.1$ for all models. 
These parameters provide coverage of the nearly ungated ($\mathcal{F}(x)\approx1$) and fully gated ($\mathcal{F}(x)=0$) regimes, with a sharp but smooth transition controlled by $w$.
Empirically, the choice of $T$ has little effect on the resulting routing patterns, suggesting that the router learns to compensate for changes in this threshold. Figure~\ref{fig:app-gating_func} illustrates the gating function and its dependence on these parameters.

\begin{figure}[!h]
    \centering
    \includegraphics[width=\textwidth]{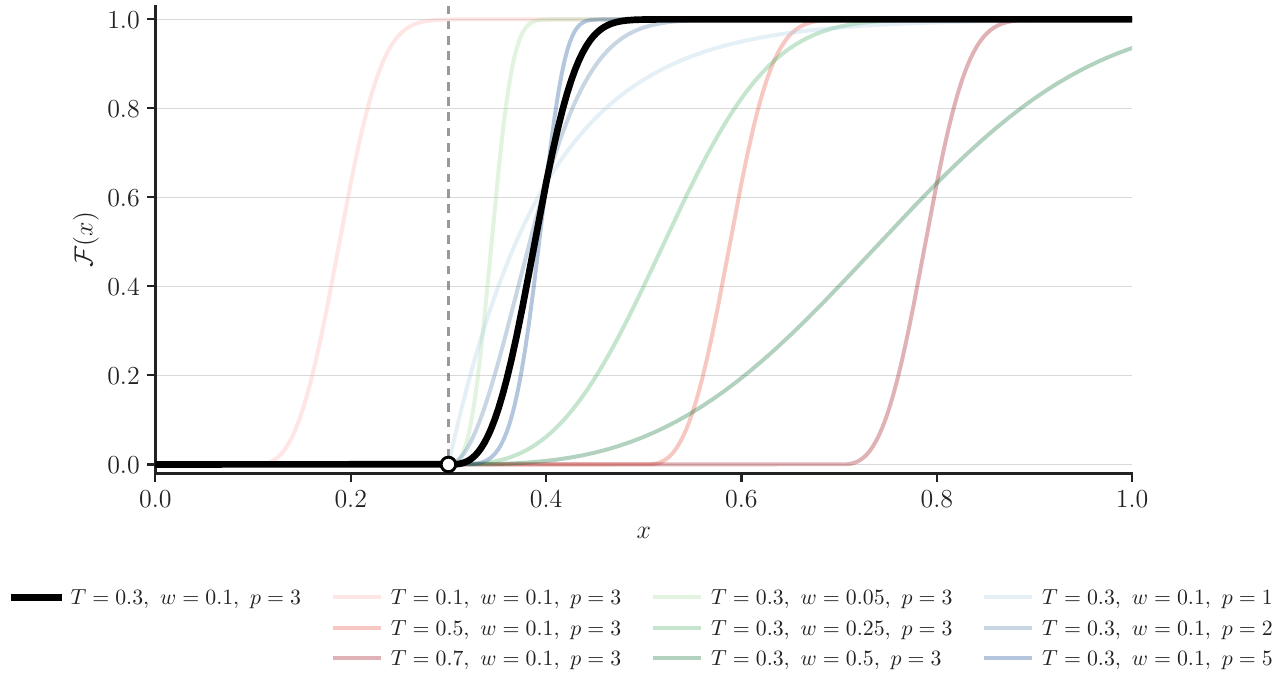}
    \caption{The continuation gate function used in Branchlegro. The black curve denotes the parameter setting used in this work.}
    \label{fig:app-gating_func}
\end{figure}

We note the continuation-gating function is similar to the polynomial envelope function used in the \texttt{nequip}-framework to smoothly transition edge activations to zero at the radial cutoff of the model.

\section{Theoretical guarantees} \label{app:theory}

Equivariant and energy-conserving MLIPs are constructed from operations that preserve their underlying symmetries and conservation laws.
For conservative MLIPs, forces are obtained as derivatives of the predicted energy, so dynamic routing must not introduce discontinuities or remove nonzero derivative contributions. 
Likewise, operations acting on equivariant features must respect their transformation laws to preserve model equivariance.

BranchIP introduces dynamic, feature-level routing into an equivariant MLIP, 
altering which features participate in TP computation during inference. 
For this adaptive computation to preserve the model's physical properties, sparse evaluation must reproduce the corresponding dense gated energy and its derivatives, while preserving the model's equivariance and permutation symmetries.
We establish these guarantees for Branchlegro.

Let $\mathbf{R}$ denote the atomic coordinates and let $\mathcal{E}$ be the full directed edge set constructed from the neighbor graph. At layer $l$, each edge $ij\in\mathcal{E}$ has scalar features $\mathbf{x}^{ij,l}$, environment weights $w^{ij,l}$, and equivariant tensor features $\bar{\mathbf{V}}^{ij,l}$, consistent with the Allegro forward pass.
For clarity, we suppress irreducible representations and channel indices as well as environment tensor normalization, which do not affect the routing arguments below.

Define the environment tensor entering layer $l+1$ as
\begin{equation}
    \mathbf{A}^{i,l}
    =
    \sum_{k:(i,k)\in\mathcal{E}}
    w^{ik,l}\mathbf{Y}^{ik},
\end{equation}
where $\mathbf{Y}^{ik}$ denotes the spherical harmonic embedding.
The TP update is then
\begin{equation}
    \bar{\mathbf{V}}^{ij,l+1}
    =
    \mathbf{A}^{i,l}
    \otimes
    \bar{\mathbf{V}}^{ij,l}.
\end{equation}

The scalar component of the TP output is
\begin{equation}
    \mathbf{z}^{ij,l+1}
    =
    \left(
    \bar{\mathbf{V}}^{ij,l+1}
    \right)_{\ell=0},
\end{equation}
where $\ell$ denotes angular momentum.
The latent update is then
\begin{equation}
\label{eq:latent}
    \left(
    \mathbf{x}^{ij,l+1},w^{ij,l+1}
    \right)
    =
    \mathrm{MLP}_{\mathrm{latent}_{l+1}}
    \left(
    \left[
    \mathbf{x}^{ij,0},
    \ldots,
    \mathbf{x}^{ij,l},
    \mathbf{z}^{ij,l+1}
    \right]
    \right),
    \qquad ij\in\mathcal{E}.
\end{equation}
The final energy is computed from the concatenated scalar latent features,
\begin{equation} \label{eq:energy_readout}
    E(\mathbf R)
    =
    \sum_{ij\in\mathcal E}
    \mathrm{MLP}_{\mathrm{out}}
    \left(
    \left[
    \mathbf{x}^{ij,0},
    \ldots,
    \mathbf{x}^{ij,L}
    \right]
    \right).
\end{equation}

With this notation, we show that Branchlegro (i) is algebraically equivalent to dense gated evaluation, (ii) preserves energy gradients and Hessians, and therefore forces, and (iii) preserves the equivariance and permutation symmetries of Allegro.

\subsection{Energy equivalence}

For a layer equipped with a router, define the continuation gate
\begin{equation}
    f^{ij,l} = \mathcal{F}(\nu^{ij,l}),
    \qquad
    \nu^{ij,l}
    =
    \mathrm{MLP}_{\mathrm{router}_l}
    \left(\mathbf{x}^{ij,l}\right),
\end{equation}
with $f^{ij,l}\in[0,1]$. For layers without a router, we define
$f^{ij,l}=1$. The gated tensor feature is
\begin{equation}
    \mathbf{V}^{ij,l}
    =
    f^{ij,l}\bar{\mathbf{V}}^{ij,l},
\end{equation}
and the gated TP is
\begin{equation} \label{eq:tp}
    \bar{\mathbf{V}}^{ij,l+1}
    =
    \mathbf{A}^{i,l}
    \otimes
    \mathbf{V}^{ij,l},
    \qquad ij\in\mathcal{E}.
\end{equation}

For energy equivalence, we require the following:
\begin{enumerate}[label=\textbf{A\arabic*.}]
    \item \label{ass:one}
    \textbf{Exact-zero slicing.}
    An edge is removed from TP computation only when its continuation
    gate is exactly zero.

    \item \label{ass:two}
    \textbf{Dense scalar computation.}
    Scalar features $\mathbf{x}^{ij,l}$ and environment weights
    $w^{ij,l}$ remain evaluated for all edges. Only TP computation is sparsified.
\end{enumerate}

Let $E_{\mathrm{dense}}(\mathbf{R})$ denote evaluation in which all edges
remain in the TP computation at every layer, including edges with
$f^{ij,l}=0$. Thus,
\begin{equation}
    \bar{\mathbf{V}}_{\mathrm{dense}}^{ij,l+1}
    =
    \mathbf{A}_{\mathrm{dense}}^{i,l}
    \otimes
    \mathbf{V}_{\mathrm{dense}}^{ij,l},
    \qquad ij\in\mathcal{E}.
\end{equation}

Let $E_{\mathrm{sparse}}(\mathbf{R})$ denote the corresponding sparse
evaluation. Define the active edge sets recursively as
\begin{equation}
    \mathcal{E}_0=\mathcal{E},
    \qquad
    \mathcal{E}_{l+1}
    =
    \left\{
    ij\in\mathcal{E}_l:
    f^{ij,l}>0
    \right\}.
\end{equation}
Sparse evaluation computes
\begin{equation}
    \bar{\mathbf{V}}_{\mathrm{sparse}}^{ij,l+1}
    =
    \mathbf{A}_{\mathrm{sparse}}^{i,l}
    \otimes
    \mathbf{V}_{\mathrm{sparse}}^{ij,l},
    \qquad
    ij\in\mathcal{E}_{l+1},
\end{equation}
while the TP is not evaluated for
$ij\notin\mathcal{E}_{l+1}$.

Because subsequent scalar computation remains defined over the full edge
set, the TP output is zero-padded as
\begin{equation}
\label{eq:scalars}
    \mathbf{z}^{ij,l+1}_{\mathrm{sparse}}
    =
    \begin{cases}
        \left(
        \bar{\mathbf{V}}_{\mathrm{sparse}}^{ij,l+1}
        \right)_{\ell=0},
        & ij\in\mathcal{E}_{l+1},
        \\[3pt]
        \mathbf{0},
        & ij\notin\mathcal{E}_{l+1}.
    \end{cases}
\end{equation}

For dense evaluation,
$\mathbf{z}_{\mathrm{dense}}^{ij,l+1}
=(\bar{\mathbf{V}}_{\mathrm{dense}}^{ij,l+1})_{\ell=0}$.

By Assumption~\ref{ass:two}, Equation~\ref{eq:latent} is evaluated for
all $ij\in\mathcal{E}$ in both the dense and sparse evaluations.

\begin{Proposition} \label{prop:energy}
Under Assumptions~\ref{ass:one} and~\ref{ass:two}, dense gated evaluation
and sparse active-set evaluation produce identical scalar latent features
at every layer and therefore identical energies:
\begin{equation}
    E_{\mathrm{dense}}(\mathbf{R})
    =
    E_{\mathrm{sparse}}(\mathbf{R}),
    \qquad
    \forall\,\mathbf{R}.
\end{equation}
\end{Proposition}

\begin{proof}
We proceed by induction over the layers. At layer $l$, assume that
\begin{equation}
    \mathbf{x}_{\mathrm{dense}}^{ij,l}
    =
    \mathbf{x}_{\mathrm{sparse}}^{ij,l},
    \qquad
    w_{\mathrm{dense}}^{ij,l}
    =
    w_{\mathrm{sparse}}^{ij,l},
    \qquad
    \forall\,ij\in\mathcal{E},
\end{equation}
and that
\begin{equation}
    \bar{\mathbf{V}}_{\mathrm{dense}}^{ij,l}
    =
    \bar{\mathbf{V}}_{\mathrm{sparse}}^{ij,l},
    \qquad
    \forall\,ij\in\mathcal{E}_l,
\end{equation}
while
\begin{equation}
    \bar{\mathbf{V}}_{\mathrm{dense}}^{ij,l}
    =
    \mathbf{0},
    \qquad
    \forall\,ij\notin\mathcal{E}_l.
\end{equation}
These conditions hold initially because
$\mathcal{E}_0=\mathcal{E}$ and both evaluations begin from the same
embedding.

Since the scalar features are identical, both evaluations produce the
same continuation gates for all active edges. Moreover, by
Assumption~\ref{ass:two}, the environment weights $w^{ij,l}$ remain
evaluated for every edge in both evaluations. Since these weights are
identical by the induction hypothesis and the environment tensor is
constructed over the full edge set,
\begin{equation}
    \mathbf{A}_{\mathrm{dense}}^{i,l}
    =
    \sum_{k:(i,k)\in\mathcal E}
    w_{\mathrm{dense}}^{ik,l}\mathbf Y^{ik}
    =
    \sum_{k:(i,k)\in\mathcal E}
    w_{\mathrm{sparse}}^{ik,l}\mathbf Y^{ik}
    =
    \mathbf{A}_{\mathrm{sparse}}^{i,l}.
\end{equation}

For an edge $ij\in\mathcal{E}_{l+1}$, both evaluations therefore apply
the same gate to the same tensor feature and evaluate the same TP, giving by Equation~\ref{eq:tp}
\begin{equation}
    \bar{\mathbf{V}}_{\mathrm{dense}}^{ij,l+1}
    =
    \bar{\mathbf{V}}_{\mathrm{sparse}}^{ij,l+1}.
\end{equation}

Now consider $ij\notin\mathcal{E}_{l+1}$. If the edge exits at layer
$l$, then by Assumption~\ref{ass:one}, $f^{ij,l}=0$, and hence
\begin{equation}
    \mathbf{V}_{\mathrm{dense}}^{ij,l}
    =
    f^{ij,l}\bar{\mathbf{V}}_{\mathrm{dense}}^{ij,l}
    =
    \mathbf{0}.
\end{equation}
By bilinearity of the TP,
\begin{equation}
    \bar{\mathbf{V}}_{\mathrm{dense}}^{ij,l+1}
    =
    \mathbf{A}_{\mathrm{dense}}^{i,l}
    \otimes \mathbf{0}
    =
    \mathbf{0}.
\end{equation}
If the edge exited at an earlier layer, the induction hypothesis already gives $\bar{\mathbf{V}}_{\mathrm{dense}}^{ij,l}=\mathbf{0}$, and the same conclusion follows. 
Thus, dense evaluation produces exactly zero for every TP output omitted by sparse evaluation. After zero-padding (Equation~\ref{eq:scalars}),
\begin{equation}
    \mathbf{z}_{\mathrm{dense}}^{ij,l+1}
    =
    \mathbf{z}_{\mathrm{sparse}}^{ij,l+1},
    \qquad
    \forall\,ij\in\mathcal E.
\end{equation}

By Assumption~\ref{ass:two}, the latent update in Equation~\ref{eq:latent} is evaluated for every edge in both implementations. 
Since its complete input is identical, it follows that
\begin{equation}
    \mathbf{x}_{\mathrm{dense}}^{ij,l+1}
    =
    \mathbf{x}_{\mathrm{sparse}}^{ij,l+1},
    \qquad
    w_{\mathrm{dense}}^{ij,l+1}
    =
    w_{\mathrm{sparse}}^{ij,l+1},
    \qquad
    \forall\,ij\in\mathcal E.
\end{equation}
This establishes the induction step.

Therefore the complete scalar latent histories are identical in dense
and sparse evaluation. By the common energy readout defined in
Equation~\ref{eq:energy_readout},
\begin{equation}
    E_{\mathrm{dense}}(\mathbf R)
    =
    E_{\mathrm{sparse}}(\mathbf R),
    \qquad
    \forall\,\mathbf R.
\end{equation}
\end{proof}

\subsection{Derivative equivalence}

To extend the equivalence to energy gradients, we additionally require:
\begin{enumerate}[label=\textbf{A\arabic*.},start=3]
    \item \label{ass:three}
    \textbf{Smooth routing.}
    The router probability map $\nu^{ij,l}$ and the continuation function $\mathcal{F}$ are $C^2$, and the zero set
    \begin{equation}
        \mathcal{Z}
        \coloneqq
        \{u:\mathcal{F}(u)=0\}
    \end{equation}
    is an interval with nonempty interior.
\end{enumerate}

The interior condition is essential: a continuation function vanishing only at an isolated point would satisfy $\mathcal{F}(\nu^\star)=0$ but generally $\mathcal{F}'(\nu^\star)\neq 0$, leaving a nonzero derivative contribution in the omitted branch.

\begin{Proposition} \label{prop:forces}
Under Assumptions~\ref{ass:one}--\ref{ass:three}, differentiating the sparse Branchlegro computation yields the same energy gradients and Hessians as dense gated evaluation. Consequently,
\begin{equation}
    \mathbf{F}_{\mathrm{sparse}}(\mathbf{R})
    = -\nabla_{\mathbf{R}} E_{\mathrm{sparse}}(\mathbf{R})
    = -\nabla_{\mathbf{R}} E_{\mathrm{dense}}(\mathbf{R})
    = \mathbf{F}_{\mathrm{dense}}(\mathbf{R}),
    \qquad
    \nabla^2_{\mathbf{R}} E_{\mathrm{sparse}} = \nabla^2_{\mathbf{R}} E_{\mathrm{dense}}.
\end{equation}
\end{Proposition}

\begin{proof}
    By Proposition~\ref{prop:energy}, under
    Assumptions~\ref{ass:one} and~\ref{ass:two}, the dense and sparse
    energy functions are pointwise identical. It therefore remains to show
    that differentiation through the sparse computation does not omit a
    nonzero derivative contribution when the active set changes.

    Consider one gated tensor feature,
    \begin{equation}
        \mathbf{V}(\mathbf{R})
        =
        f(\mathbf{R})\bar{\mathbf{V}}(\mathbf{R})
        =
        \mathcal{F}(\nu(\mathbf{R}))
        \bar{\mathbf{V}}(\mathbf{R}).
    \end{equation}
    By Assumption~\ref{ass:three} and the twice continuous differentiability of the Allegro backbone, the gated feature is twice differentiable with respect to $\mathbf{R}$.
    By the chain rule,
    \begin{equation}
    \label{eq:tensor_grad}
    \begin{aligned}
        \nabla_{\mathbf{R}}\mathbf{V}
        &= \mathcal{F}'(\nu)\,(\nabla_{\mathbf{R}}\nu\otimes\bar{\mathbf{V}})
         + \mathcal{F}(\nu)\,\nabla_{\mathbf{R}}\bar{\mathbf{V}},
        \\
        \nabla^2_{\mathbf{R}}\mathbf{V}
        &= \mathcal{F}''(\nu)\,
           (\nabla_{\mathbf{R}}\nu\otimes\nabla_{\mathbf{R}}\nu\otimes\bar{\mathbf{V}})
        \\
        &\quad
         + \mathcal{F}'(\nu)\,
           \big(
           \nabla^2_{\mathbf{R}}\nu\otimes\bar{\mathbf{V}}
           + 2\,\nabla_{\mathbf{R}}\nu\otimes\nabla_{\mathbf{R}}\bar{\mathbf{V}}
           \big)
         + \mathcal{F}(\nu)\,\nabla^2_{\mathbf{R}}\bar{\mathbf{V}}.
    \end{aligned}
    \end{equation}
    Let $\mathcal{Z}$ be the zero-valued slicing interval defined in
    Assumption~\ref{ass:three}. We consider three cases.

    If $\nu\notin\mathcal{Z}$, then $\mathcal{F}(\nu)\neq0$ by definition of $\mathcal{Z}$. 
    By Assumption~\ref{ass:one}, no sparse slicing occurs, so the dense and sparse evaluations execute the same differentiable operations.

    If $\nu\in\operatorname{int}(\mathcal{Z})$, then $\mathcal{F}$ is locally constant and $\mathcal{F}(\nu)=0$, $\mathcal{F}'(\nu)=0$, and $\mathcal{F}''(\nu)=0$.
    Therefore, by Equation~\ref{eq:tensor_grad}, $\nabla_{\mathbf{R}}\mathbf{V}=\mathbf{0}$ and $\nabla^2_{\mathbf{R}}\mathbf{V}=\mathbf{0}$.
    Thus, a branch omitted by sparse evaluation contributes neither to the forward value nor to its derivatives.

    Finally, let $\nu^\star$ lie on the boundary of $\mathcal{Z}$.
    Since $\mathcal{F}$ is continuous, $\mathcal{Z}$ is closed and
    $\mathcal{F}(\nu^\star)=0$. By Assumption~\ref{ass:three},
    $\mathcal{F}'$ and $\mathcal{F}''$ are continuous, and both vanish
    identically on $\operatorname{int}(\mathcal{Z})$, where $\mathcal{F}$ is
    constant. Therefore,
    \begin{equation}
        \mathcal{F}'(\nu^\star)
        = \lim_{\substack{u\to\nu^\star\\u\in\operatorname{int}(\mathcal{Z})}}
        \mathcal{F}'(u) = 0,
        \qquad
        \mathcal{F}''(\nu^\star)
        = \lim_{\substack{u\to\nu^\star\\u\in\operatorname{int}(\mathcal{Z})}}
        \mathcal{F}''(u) = 0.
    \end{equation}
    Consequently, every term in Equation~\ref{eq:tensor_grad} vanishes, so
    $\nabla_{\mathbf{R}}\mathbf{V}=\mathbf{0}$ and
    $\nabla^2_{\mathbf{R}}\mathbf{V}=\mathbf{0}$.
    Hence, at the boundary where the sparse active set changes, the omitted
    branch has zero value, zero first derivative, and zero second derivative.

    Thus, changing the sparse active set removes no nonzero first- or second-derivative contribution. All remaining downstream operations are shared between dense and sparse evaluation and are $C^2$ in Allegro, so by the chain rule their first and second derivatives with respect to $\mathbf{R}$ coincide. Together with Proposition~\ref{prop:energy}, this gives
    \begin{equation}
    \nabla_{\mathbf{R}}E_{\mathrm{dense}}(\mathbf{R})
    =
    \nabla_{\mathbf{R}}E_{\mathrm{sparse}}(\mathbf{R}),
    \qquad
    \nabla^2_{\mathbf{R}}E_{\mathrm{dense}}(\mathbf{R})
    =
    \nabla^2_{\mathbf{R}}E_{\mathrm{sparse}}(\mathbf{R}),
    \end{equation}
    and hence
    \begin{equation}
    \mathbf{F}_{\mathrm{dense}}(\mathbf{R})
    =
    \mathbf{F}_{\mathrm{sparse}}(\mathbf{R}).
    \end{equation}
\end{proof}

The same argument extends directly to stresses by replacing derivatives with respect to atomic coordinates by derivatives with respect to strain.

\begin{Corollary}\label{cor:regularity}
Under Assumptions~\ref{ass:one}--\ref{ass:three}, the Branchlegro PES is $C^2$ in $\mathbf{R}$.
\end{Corollary}

\begin{proof}
By Assumption~\ref{ass:three}, $\mathcal{F}$ and $\nu$ are $C^2$, and the Allegro backbone is $C^2$. 
The dense gated energy is therefore $C^2$ as a composition of $C^2$ maps. By Proposition~\ref{prop:energy}, $E_{\mathrm{sparse}} = E_{\mathrm{dense}}$ pointwise, so the sparse energy inherits this regularity.
\end{proof}

The same argument applies at any order $m$ for which $\mathcal{F}$, $\nu$, and the backbone are $C^m$, and derivatives up to order $m$ computed through the sparse path are exact.

\subsection{Symmetry preservation}

\begin{Proposition}\label{prop:symmetry}
    Given an $E(3)$-invariant continuation gate produced by a shared edge-wise router, Branchlegro preserves 
    (i) the $E(3)$-equivariance of Allegro's tensor features and invariance of its energy,
    and (ii) its permutation symmetry.
\end{Proposition}
\begin{proof}
    Let $g\in O(3)$ denote a rotation or reflection acting on the atomic configuration $\mathbf R$, and let $D[g]$ denote the corresponding representation acting on Allegro's equivariant tensor features.
    Translations act trivially since Allegro depends only on relative geometry. 
    Let $\pi$ denote a permutation of atom indices, acting on edge labels as $ij\mapsto\pi(i)\pi(j)$.
    
\textbf{$E(3)$-equivariance.}
We proceed inductively over layers. Initially,
$\mathcal{E}_0(\mathbf{R})=\mathcal{E} = \mathcal{E}_0(g\mathbf{R})$.
Assume at layer $l$ that the active sets coincide and that
$\bar{\mathbf V}^{ij,l}$ is equivariant. 
By invariance of the continuation gate,
$f^{ij,l}(g\mathbf R)=f^{ij,l}(\mathbf R)$.
Thus,
\begin{equation}
\begin{aligned}
\mathbf V^{ij,l}(g\mathbf R)
&= f^{ij,l}(g\mathbf R)\bar{\mathbf V}^{ij,l}(g\mathbf R) \\
&= f^{ij,l}(\mathbf R)D[g]\bar{\mathbf V}^{ij,l}(\mathbf R) \\
&= D[g]\mathbf V^{ij,l}(\mathbf R).
\end{aligned}
\end{equation}
Moreover,
\begin{equation}
\begin{aligned}
\mathcal E_{l+1}(g\mathbf R)
&= \{ij\in\mathcal E_l(g\mathbf R):
f^{ij,l}(g\mathbf R)>0\} \\
&= \{ij\in\mathcal E_l(\mathbf R):
f^{ij,l}(\mathbf R)>0\} 
= \mathcal E_{l+1}(\mathbf R).
\end{aligned}
\end{equation}

The subsequent Allegro operations are equivariant, so $\bar{\mathbf V}^{ij,l+1}$ retains the corresponding transformation law. The result therefore follows by induction, with the scalar energy readout remaining $E(3)$-invariant.

\textbf{Permutation symmetry.}
We proceed inductively over layers. Initially,
$\mathcal E_0(\pi\mathbf R) = \pi\!\left(\mathcal E_0(\mathbf R)\right)$.
Assume at layer $l$ that the active sets and incoming tensor features
permute consistently,
\begin{equation*}
\mathcal E_l(\pi\mathbf R)
=
\pi\!\left(\mathcal E_l(\mathbf R)\right),
\qquad
\bar{\mathbf V}^{\pi(i)\pi(j),l}(\pi\mathbf R)
=
\bar{\mathbf V}^{ij,l}(\mathbf R).
\end{equation*}
Since Allegro's edge features permute consistently and the same router is applied to every edge,
$f^{\pi(i)\pi(j),l}(\pi\mathbf R) = f^{ij,l}(\mathbf R)$. 
Thus,
\begin{equation}
\begin{aligned}
\mathbf V^{\pi(i)\pi(j),l}(\pi\mathbf R)
&=
f^{\pi(i)\pi(j),l}(\pi\mathbf R)
\bar{\mathbf V}^{\pi(i)\pi(j),l}(\pi\mathbf R) \\
&=
f^{ij,l}(\mathbf R)
\bar{\mathbf V}^{ij,l}(\mathbf R) \\
&=
\mathbf V^{ij,l}(\mathbf R).
\end{aligned}
\end{equation}
Moreover,
\begin{equation}
\begin{aligned}
\mathcal E_{l+1}(\pi\mathbf R)
&=
\{\pi(i)\pi(j)\in\mathcal E_l(\pi\mathbf R):
f^{\pi(i)\pi(j),l}(\pi\mathbf R)>0\} \\
&=
\pi\!\left(
\{ij\in\mathcal E_l(\mathbf R):
f^{ij,l}(\mathbf R)>0\}
\right) \\
&=
\pi\!\left(\mathcal E_{l+1}(\mathbf R)\right).
\end{aligned}
\end{equation}
The subsequent Allegro operations preserve permutation symmetry, so
$\bar{\mathbf V}^{ij,l+1}$ permutes consistently as well. The result
therefore follows by induction, and the final sum over edge energies is
permutation invariant.
\end{proof}
In Branchlegro, each router is an MLP acting only on Allegro's invariant scalar edge features $\mathbf{x}^{ij,l}$. 
Therefore its output $\nu^{ij,l}$, and consequently the continuation gate $f^{ij,l}=\mathcal{F}(\nu^{ij,l})$, are $E(3)$-invariant.

Together, Propositions~\ref{prop:energy}--\ref{prop:symmetry} and Corollary~\ref{cor:regularity} establish Theorem~\ref{theo:main}: Assumptions~\ref{ass:one} and~\ref{ass:two} state condition (ii) for Allegro's edge-wise routing, Assumption~\ref{ass:three} states the smoothness in (i), and the invariance in (i) follows from the routers acting on invariant scalar features.

\section{Training Setup} \label{app:setup}

All model training was conducted on a combination of  NVIDIA A100-SXM4-40GB, NVIDIA A100-SXM4-80GB, and H200 GPUs. For fine-tuning the FMs, the models were compiled and employed cuEquivariance TP kernels.
The hyperparameters of the bespoke Allegro models for CO/Pt and CDP are summarized in Table~\ref{tab:copt_architecture} and Table~\ref{tab:cdp_architecture}. The corresponding hyperparameters of the Allegro and Branchlegro training are reported in Tables~\ref{tab:copt_allegro_training}--\ref{tab:copt_branchlegro_training} for CO/Pt and Tables~\ref{tab:cdp_allegro_training}--\ref{tab:cdp_branchlegro_training} for CDP.
Across Branchlegro runs, the random seed controls router initialization and data-batch ordering; the teacher model and train/validation/test split were held fixed.

\begin{table*}[tbhp]
\centering
\caption{Bespoke model architecture and training hyperparameters used for CO/Pt.}
\label{tab:copt_hyperparams}
\small

\begin{subtable}[t]{0.45\textwidth}
\centering
\caption{Architecture~\citep{owen_surface_2024}.}
\label{tab:copt_architecture}

{\footnotesize\scshape
\begin{tabularx}{\linewidth}{@{}>{\raggedright\arraybackslash}X l@{}}
  \toprule
  Hyperparameter & Value \\
  \midrule
  Cutoff radius & $8.0\,\mathrm{\AA}$ \\
  Number of layers & $1\;\; / \;\;2$ \\
  $l_{\max}$ & $4$ \\
  Parity & True \\
  Number of scalar features & $256$ \\
  Number of tensor features & $16$ \\
  Radial Bessel basis functions & $8$ \\
  Trainable Bessel basis & False \\
  Polynomial cutoff exponent & $6$ \\
  Radial-chemical embedding dim & $64$ \\
  Scalar embedding MLP depth & $1$ \\
  Scalar embedding MLP width & $128$ \\
  Scalar embedding nonlinearity & SiLU \\
  Allegro MLP depth & $3$ \\
  Allegro MLP width & $256$ \\
  Allegro MLP nonlinearity & SiLU \\
  TP path-channel coupling & False \\
  Readout MLP depth & $1$ \\
  Readout MLP width & $256$ \\
  Readout MLP nonlinearity & SiLU \\
  Per-type energy scales trainable & False \\
  Per-type energy shifts trainable & False \\
  \bottomrule
\end{tabularx}
}
\end{subtable}
\hfill
\begin{minipage}[t]{0.50\textwidth}

\begin{subtable}[t]{\linewidth}
\centering
\caption{Allegro training.}
\label{tab:copt_allegro_training}

{\footnotesize\scshape
\begin{tabularx}{\linewidth}{@{}>{\raggedright\arraybackslash}X l@{}}
  \toprule
  Hyperparameter & Value \\
  \midrule
  Epochs & $1000$ \\
  Optimizer & Adam \\
  Learning rate & $2\times10^{-3}$ \\
  LR scheduler & Reduce on Plateau \\
  Early stopping & Patience $300$ \\
  Force loss coeff. & $1$ \\
  Energy loss coeff. & $1$ \\
  Stress loss coeff. & $1$ \\

  Batch size & $2$\\
  \bottomrule
\end{tabularx}
}
\end{subtable}

\vspace{1em}

\begin{subtable}[t]{\linewidth}
\centering
\caption{Branchlegro distillation.}
\label{tab:copt_branchlegro_training}

{\footnotesize\scshape
\begin{tabularx}{\linewidth}{@{}>{\raggedright\arraybackslash}X l@{}}
  \toprule
  Hyperparameter & Value \\
  \midrule
  Epochs & $100$ \\
  Optimizer & Adam \\
  Backbone learning rate & $1\times10^{-3}$ \\
  Router learning rate & $5\times10^{-4}$ \\

  LR scheduler & Cosine Annealing \\
  
  Force loss coeff. & $1$\\
  Energy loss coeff. & $1$\\
  Stress loss coeff. & $1$\\

  $\mathcal{L}_{\mathrm{distill}}$ coeff. & $0.05$ \\

  Continue tolerance & $0:0.060,\ 1: 0.095$ \\
  Stop tolerance & $0:0.12,\ 1:0.155$ \\

  Router MLP depth & $2$ \\
  Router MLP width & $8$ \\

  Batch size & $8$\\
  \bottomrule
\end{tabularx}
}
\end{subtable}

\end{minipage}

\end{table*}

\begin{table*}[tbhp]
\small
\centering
\caption{FM fine-tuning hyperparameters used for CO/Pt (architecture from \cite{kavanagh_fast_2026}).}
\label{tab:copt_FM}

\begin{subtable}[t]{0.45\textwidth}
\centering
\caption{Allegro-OAM-L fine-tuning.}
\label{tab:copt_allegro_tuning}

{\footnotesize\scshape
\begin{tabularx}{\linewidth}{@{}>{\raggedright\arraybackslash}X l@{}}
  \toprule
  Hyperparameter & Value \\
  \midrule
  Epochs & $1000$ \\
  Optimizer & AdamW \\
  Learning rate & $4.1\times10^{-4}$ \\
  Weight decay & $3.4\times10^{-7}$ \\
  LR scheduler & Cosine Annealing \\
  Early stopping & Patience $600$ \\
  Force loss coeff. & $1$ \\
  Energy loss coeff. & $1$ \\
  Stress loss coeff. & $1$ \\

  Batch size & $5$\\
  \bottomrule
\end{tabularx}
}
\end{subtable}
\hfill
\begin{subtable}[t]{0.50\textwidth}
\centering
\caption{Branchlegro distillation.}
\label{tab:copt_branchlegro_tuning}

{\footnotesize\scshape
\begin{tabularx}{\linewidth}{@{}>{\raggedright\arraybackslash}X l@{}}
  \toprule
  Hyperparameter & Value \\
  \midrule
  Epochs & $200$ \\
  Optimizer & Adam \\
  Backbone learning rate & $1\times10^{-3}$ \\
  Router learning rate & $1\times10^{-4}$ \\
  LR scheduler & Cosine Annealing \\
  
  Force loss coeff. & $1$ \\
  Energy loss coeff. & $1$ \\
  Stress loss coeff. & $1$ \\
  $\mathcal{L}_{\mathrm{distill}}$ coeff. & $0.1$ \\

  Continue tolerance &
    \makecell[l]{$0:0.030,\ 1:0.042,$\\$2:0.048,\ 3:0.060$} \\
  Stop tolerance &
    \makecell[l]{$0:0.060,\ 1:0.071,$\\$2:0.077,\ 3:0.090$} \\

  Router MLP depth & $1$ \\
  Router MLP width & $8$ \\

  Batch size & $10$\\
  \bottomrule
\end{tabularx}
}
\end{subtable}

\end{table*}

\begin{table*}[tbhp]
\centering
\caption{Bespoke model architecture and training hyperparameters used for CDP.}
\label{tab:cdp_hyperparams}

\begin{subtable}[t]{0.45\textwidth}
\centering
\caption{Architecture~\citep{wang_revealing_2025}.}
\label{tab:cdp_architecture}

{\footnotesize\scshape
\begin{tabularx}{\linewidth}{@{}>{\raggedright\arraybackslash}X l@{}}
  \toprule
  Hyperparameter & Value \\
  \midrule
  Cutoff radius & $7.0\,\mathrm{\AA}$ \\
  Number of layers & $1\;\; / \;\;2$ \\
  $l_{\max}$ & $2$ \\
  Parity & True \\
  Number of scalar features & $32$ \\
  Number of tensor features & $32$ \\
  Radial Bessel basis functions & $8$ \\
  Trainable Bessel basis & False \\
  Polynomial cutoff exponent & $6$ \\
  Radial-chemical embedding dim & $32$ \\
  Scalar embedding MLP depth & $2$ \\
  Scalar embedding MLP width & $128$ \\
  Scalar embedding nonlinearity & SiLU \\
  Allegro MLP depth & $2$ \\
  Allegro MLP width & $128$ \\
  Allegro MLP nonlinearity & SiLU \\
  TP path-channel coupling & False \\
  Readout MLP depth & $1$ \\
  Readout MLP width & $32$ \\
  Readout MLP nonlinearity & None \\
  Per-type energy scales trainable & True \\
  Per-type energy shifts trainable & True \\
  \bottomrule
\end{tabularx}
}
\end{subtable}
\hfill
\begin{minipage}[t]{0.50\textwidth}

\begin{subtable}[t]{\linewidth}
\centering
\caption{Allegro training~\citep{harari_beyond_2026}.}
\label{tab:cdp_allegro_training}

{\footnotesize\scshape
\begin{tabularx}{\linewidth}{@{}>{\raggedright\arraybackslash}X l@{}}
  \toprule
  Hyperparameter & Value \\
  \midrule
  Epochs & 1000 \\
  Optimizer & AdamW\\
  Learning rate & $3 \times 10^{-3}$\\
  Weight decay & $1 \times 10^{-4}$\\
  LR scheduler & Cosine Annealing \\
  Early stopping & Patience $600$ \\
  
  Force loss coeff. & $10$\\
  Energy loss coeff. & $1$\\

  Batch size & $5$\\
  \bottomrule
\end{tabularx}
}
\end{subtable}

\vspace{1em}

\begin{subtable}[t]{\linewidth}
\centering
\caption{Branchlegro distillation.}
\label{tab:cdp_branchlegro_training}

{\footnotesize\scshape
\begin{tabularx}{\linewidth}{@{}>{\raggedright\arraybackslash}X l@{}}
  \toprule
  Hyperparameter & Value \\
  \midrule
  Epochs & $500$ \\
  Optimizer & Adam \\
  Backbone learning rate & $5\times10^{-3}$ \\
  Router learning rate & $5\times10^{-4}$ \\

  LR scheduler & Cosine Annealing \\
  
  Force loss coeff. & $1$\\
  Energy loss coeff. & $1$\\
  $\mathcal{L}_{\mathrm{distill}}$ coeff. & $0.1$ \\

  Continue tolerance & $0:0.05,\ 1:0.08$ \\
  Stop tolerance & $0:0.08,\ 1:0.12$ \\

  Router MLP depth & $1$ \\
  Router MLP width & $8$ \\

  Batch size & $5$\\
  \bottomrule
\end{tabularx}
}
\end{subtable}

\end{minipage}

\end{table*}

\begin{table*}[tbhp]
\centering
\caption{FM fine-tuning hyperparameters used for CDP (architecture from \cite{kavanagh_fast_2026}).}
\label{tab:cdp_FM}

\begin{subtable}[t]{0.45\textwidth}
\centering
\caption{Allegro-OAM-L fine-tuning.}
\label{tab:cdp_allegro_tuning}

{\footnotesize\scshape
\begin{tabularx}{\linewidth}{@{}>{\raggedright\arraybackslash}X l@{}}
  \toprule
  Hyperparameter & Value \\
  \midrule
  Epochs & $1000$ \\
  Optimizer & AdamW \\
  Learning rate & $1.0\times10^{-3}$ \\
  Weight decay & $2.3\times10^{-7}$ \\
  LR scheduler & Cosine Annealing \\
  Early stopping & Patience $600$ \\
  Force loss coeff. & $10$ \\
  Energy loss coeff. & $1$ \\

  Batch size & $5$\\
  \bottomrule
\end{tabularx}
}
\end{subtable}
\hfill
\begin{subtable}[t]{0.50\textwidth}
\centering
\caption{Branchlegro distillation.}
\label{tab:cdp_branchlegro_tuning}

{\footnotesize\scshape
\begin{tabularx}{\linewidth}{@{}>{\raggedright\arraybackslash}X l@{}}
  \toprule
  Hyperparameter & Value \\
  \midrule
  Epochs & $500$ \\
  Optimizer & Adam \\
  Backbone learning rate & $5\times10^{-4}$ \\
  Router learning rate & $5\times10^{-5}$ \\
  LR scheduler & Cosine Annealing \\

  Force loss coeff. & $1$ \\
  Energy loss coeff. & $1$ \\
  $\mathcal{L}_{\mathrm{distill}}$ coeff. & $0.05$ \\

  Continue tolerance &     
     \makecell[l]{$0:0.05,\ 1:0.07,$\\$2:0.09,\ 3:0.11$} \\
  Stop tolerance &
    \makecell[l]{$0:0.10,\ 1:0.12,$\\$2:0.14,\ 3:0.16$} \\

  Router MLP depth & $1$ \\
  Router MLP width & $8$ \\

  Batch size & $4$\\
  \bottomrule
\end{tabularx}
}
\end{subtable}

\end{table*}

\pagebreak

\section{MD throughput and memory measurements}\label{app:md_throughput}

Molecular dynamics throughput was benchmarked in LAMMPS \citep{thompson_lammps_2022, johansson_lammps-kokkos_2025} using NVT simulations at 525~K with a 0.5~fs timestep for CDP and 300~K with a 1~fs timestep for CO/Pt. 
Each benchmark used an initial $10^4$-step warm-up period followed by a production run of up to $9.99\times10^6$ steps. 
LAMMPS performance statistics were recorded every 100 steps, and the benchmark run was limited to 1200~s to keep comparisons consistent across models and system sizes. 
Throughput was evaluated from the reported simulation performance after the warm-up period.
Pareto benchmarks were run on NVIDIA A100-SXM4-40GB; all remaining benchmarks were run on NVIDIA A100-SXM4-80GB GPUs. 
Models were compiled using \texttt{torch.compile}.

GPU memory was measured using PyTorch's CUDA memory profiler. Before each force evaluation in LAMMPS, the CUDA peak memory counters were reset. After the force evaluation, the peak allocated GPU memory were recorded. Reported values correspond to the maximum over all force evaluations during the simulation.
Models were compiled using \texttt{torch.compile}.

\section{Physical Fidelity} \label{app:physical_fidelity}

\subsection{Energy profile for CO/Pt}

Figure~\ref{fig:app-copt_eos} includes energy profiles of the Allegro and Branchlegro CO/Pt MLIPs against the density functional theory (DFT) ground truth. This includes energy-volume equation of state of bulk Pt and an isolated-CO bond-stretch curve. The DFT calculations followed the methodology reported by \citet{owen_surface_2024}, using VASP with PAW potentials, the vdW-DF exchange-correlation functional, and a 500 eV plane-wave cutoff. The Pt validation supercell used a 3×3×3 supercell of a conventional FCC cell with an optimized Pt lattice constant of approximately 4.03 Å. We used metallic sampling and smearing for Pt, molecular sampling and dipole corrections for CO, and consistent electronic convergence settings. 

\begin{figure}[!h]
    \centering
    \includegraphics[width=\textwidth]{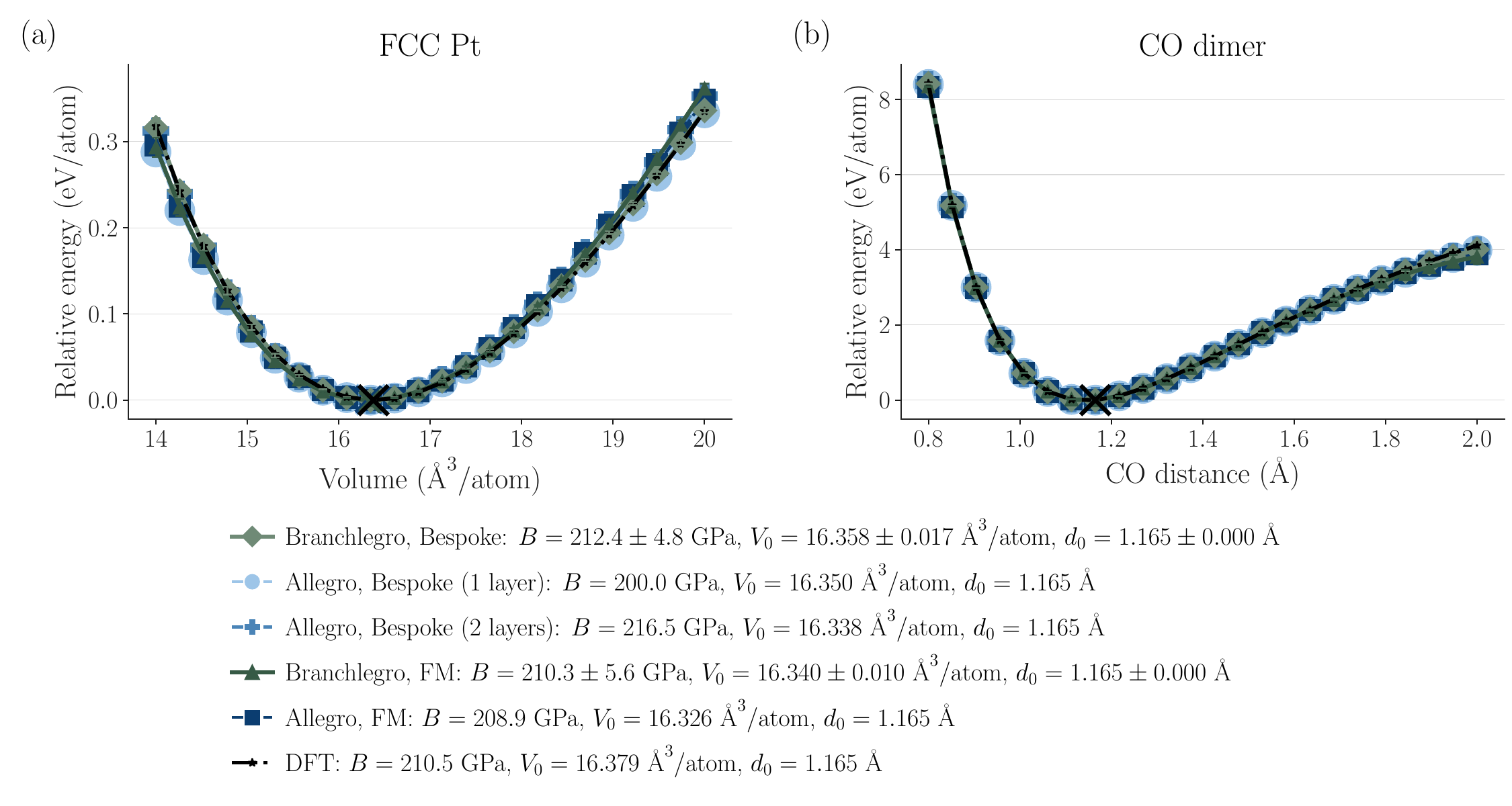}
    \caption{(a) Equation of state of fcc Pt, energy per atom as a function of atomic volume. Equilibrium volumes ($V_0$) and bulk moduli ($B$) were obtained using ASE’s stabilized-jellium equation-of-state fit. (b) Absolute energy of an isolated CO dimer as a function of bond length, with $d_0$ denoting the lowest-energy sampled distance. Crosses mark the corresponding energy minima. Branchlegro results are reported as mean $\pm$ standard deviation over seven seeds.}
    \label{fig:app-copt_eos}
\end{figure}

Overall, all models show close agreement with DFT across the physical-fidelity evaluations. 
Both Allegro and Branchlegro reproduce the relevant energy landscapes and derived material properties, indicating that adaptive routing preserves the underlying physical response of the parent models.

Using the DFT bulk modulus, $B_{\mathrm{DFT}}=210.5$ GPa, as reference, the bespoke one- and two-layer Allegro models deviate by $-4.97\%$ and $+2.85\%$. Branchlegro predicts $B=212.4$ GPa, corresponding to a smaller deviation of $+0.90\%$ than either bespoke baseline. For the FMs, both Allegro and Branchlegro remain within $1\%$ of DFT, with deviations of $-0.75\%$ and $-0.10\%$, respectively. Thus, adaptive routing overall preserves the elastic response of both bespoke and FMs.

\subsection{MD simulations for CDP activation energy} \label{app:cdp_activation}

To further validate Branchlegro dynamics for CDP, molecular dynamics (MD) simulations were performed in LAMMPS. We compare radial distribution functions (RDFs) from Branchlegro and Allegro models with \textit{ab initio} MD (AIMD) using the DFT method used to generate the CDP training dataset \citet{wang_revealing_2025}. The AIMD RDF reference averages two 144-atom trajectories that differ only in their initial velocity seed. For the bespoke models, Branchlegro, L1, and L2 Allegro models, RDFs are computed from the average of three 525~K MD simulations that differ only in their initial velocity seed with 1000 atoms. For the more computationally expensive FM models the RDFs are computed from the average of three 144-atom trajectories at 525~K that differ only in their initial velocity seed.

\begin{figure}[!h]
    \centering
    \includegraphics[width=\textwidth]{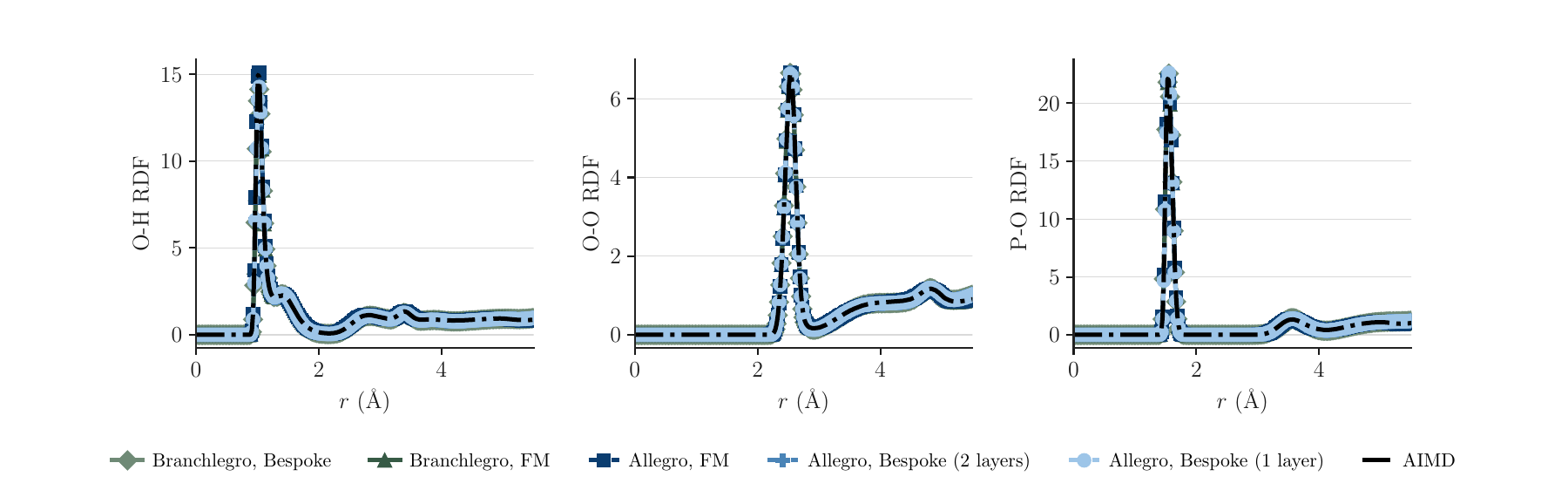}
    \caption{Radial distribution functions for O-H, O-O, and P-O pairs in molecular dynamics (MD) simulations of CDP at 525 K with \textit{ab initio} MD reference.}
    \label{fig:CDP_RDFs}
\end{figure}

Along with RDFs, we calculate the proton diffusion activation energy, a validation requiring accurate long-timescale MD dynamics. Simulations were performed in LAMMPS using the L1 Allegro, L2 Allegro, and Branchlegro bespoke interatomic potentials. Simulations were conducted in the NVT ensemble with periodic boundary conditions in all three dimensions with 1000 atoms and a Nosé–Hoover thermostat. Initial velocities were drawn from a Gaussian distribution at the target temperature. A timestep of 0.5~fs was used, and each simulation was run for $4\times10^6$ steps, corresponding to 2 ns of simulated time. Atomic configurations were recorded every 100 steps.

Hydrogen diffusion was determined from the mean-squared displacement (MSD) of H atoms averaged over three trajectories at each temperature that differ only in the initial velocity seed. The diffusion coefficient was obtained from the linear diffusive regime using the three-dimensional Einstein relation. The initial 10~ps of each trajectory were excluded from the MSD analysis. Diffusion coefficients were calculated at 525~K, 550~K, 575~K, and 625~K for each potential. Activation energies were obtained by fitting the resulting temperature-dependent diffusion coefficients to the Arrhenius relation. We restrict activation energy calculations to the bespoke models because the required long-timescale simulations are extremely expensive for the FMs.

\begin{figure}[!h]
    \centering
    \includegraphics[width=\textwidth]{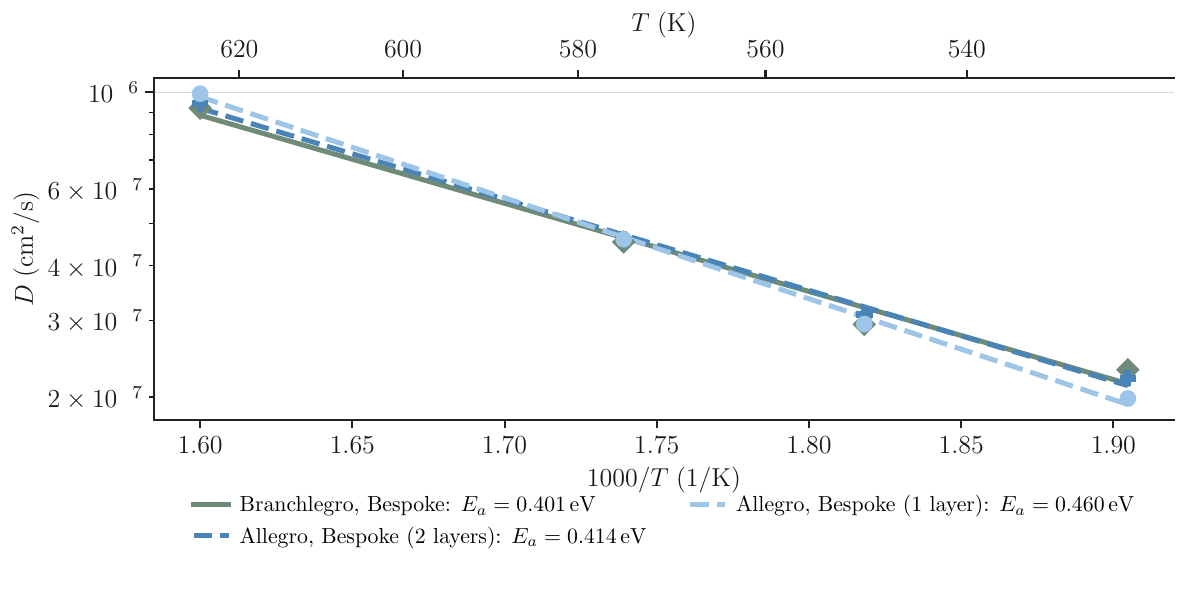}
    \caption{Arrhenius plot of diffusion coefficients for L1 Allegro, L2 Allegro, and Branchlegro bespoke models. The activation energy, reported in the legend, is determined from the slope of the fitted linear regression.}
    \label{fig:activation_energy}
\end{figure}

\section{Energy conservation in MD}\label{app:energy_conservation}

For a correct MD simulation, the total energy (summation of potential and kinetic energies) should remain constant over time in the microcanonical (NVE) ensemble. To evaluate the energy conservation of Branchlegro, we run NVE MD simulations in LAMMPS for all Branchlegro models. Simulations were initialized at 525~K and 300~K with periodic boundary conditions for CDP and CO/Pt and a timestep of 0.5~fs and 1~fs (production timesteps), and were propagated for 100~ps. The total energy was recorded every 100 steps and converted to an energy drift per atom relative to the initial value, after an initial 10~ps equilibration period. Energy conservation was assessed by monitoring this change in total energy as a function of simulation time for the different Branchlegro and Allegro models.

\begin{figure}[!h]
    \centering
    \includegraphics[width=\textwidth]{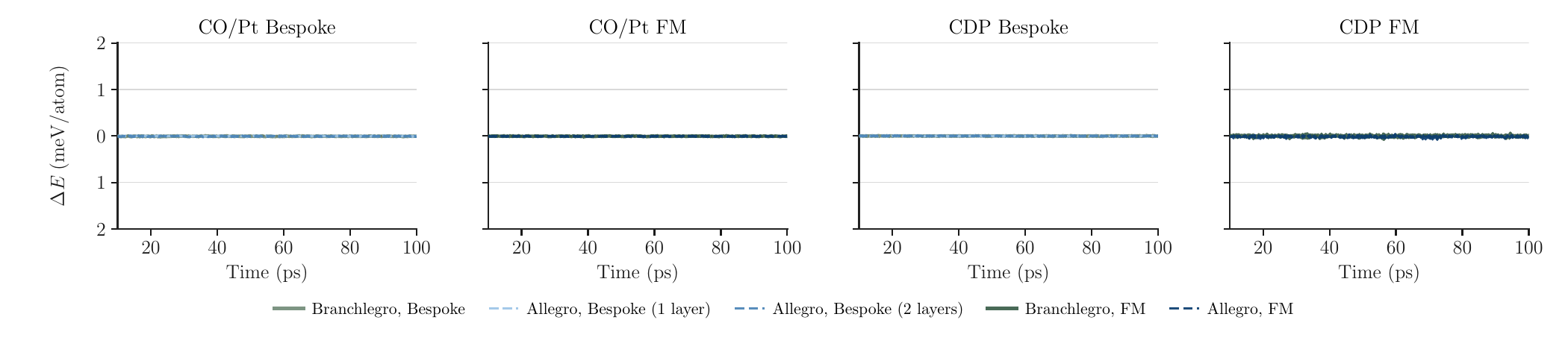}
    \caption{Energy conservation during molecular dynamics in the microcanonical ensemble. Each panel shows the drift in total energy, reported as the change in total energy in meV per atom relative to the value at 10~ps. All cases show stable and conservative dynamics.}
    \label{fig:energy_conservation}
\end{figure}

\section{Interpretability details} \label{app:interpretability}

The CO/Pt analysis utilized the following classification rules:
A Pt is adsorbate-bound if a C atom lies within $2.6$\,\AA\footnote{Determined by JmolNN from \citet{pan_benchmarking_2021} with the radii taken from \citet{cordero_covalent_2008}.}; otherwise it is bare surface if its Pt--Pt coordination within $3.2$\AA~(tail of first peak in Pt--Pt RDF, Figure~\ref{fig:routing_vis}c) is below the fcc bulk value of 12, and bulk if not.
Adsorbate-bound Pt is further split by the coordination of its adsorbed carbon under the same $2.6$\,\AA~criterion: top if that C binds a single Pt, multiple if it is shared with more than one Pt (bridge or hollow). Figure~\ref{fig:app-adsorption} shows the different adsorption sites for CO in Pt.

\begin{figure*}[!b]
  \centering
  \includegraphics[width=\textwidth]{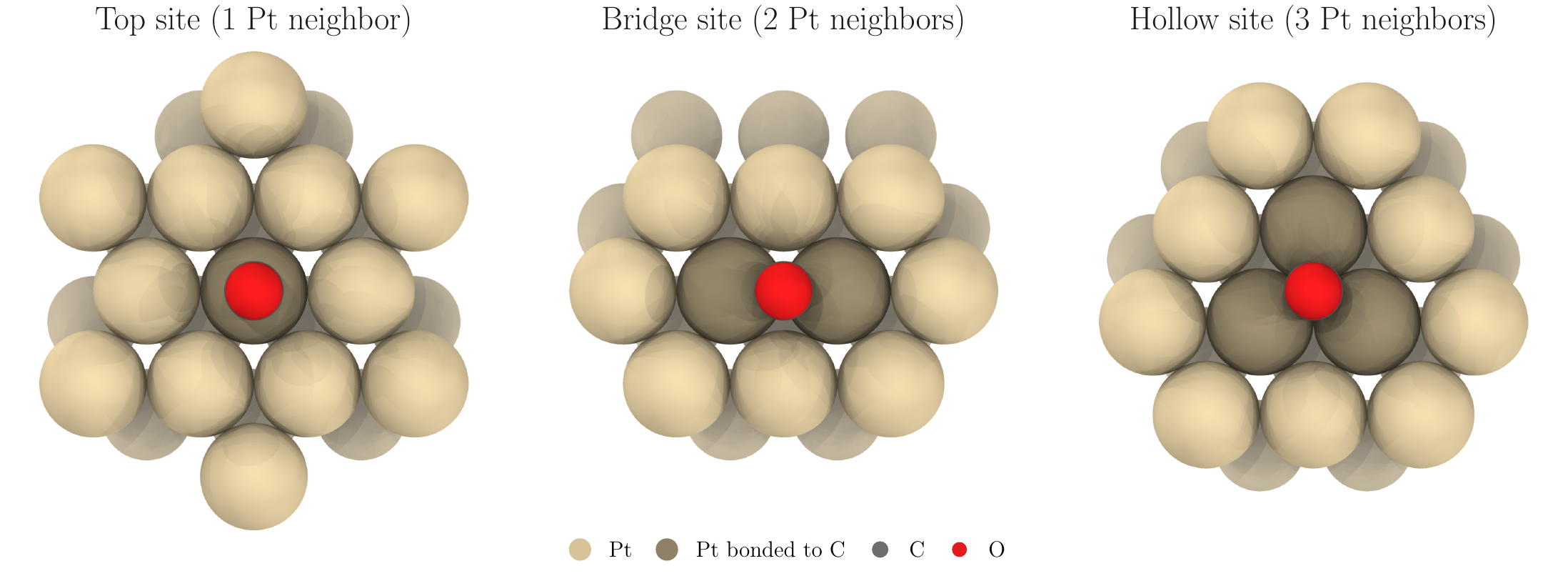} 
  \caption{Schematic of CO adsorption sites on Pt.}
  \label{fig:app-adsorption}
\end{figure*}

The CDP analysis utilized the following classification rules:
For the P--H edges, covalent was classified as P owning the H's closest O atom. Hydrogen-bonded was classified when that P owns the H's second nearest O and that O is within $2.20$\,\AA~(the minimum after the second peak in the O--H RDF). Otherwise, it was classified as unrelated. For O--O edges, same PO$_4$ was classified when both O atoms share the same closest P atom. O-pair was classified when the two O atoms are the closest and second-nearest oxygens of a  common proton, the latter again within $2.20$\,\AA, so that a proton bridges them (covalently bonded to one and hydrogen-bonded to the other).

\begin{figure*}[!b]
  \centering
  \includegraphics[width=\textwidth]{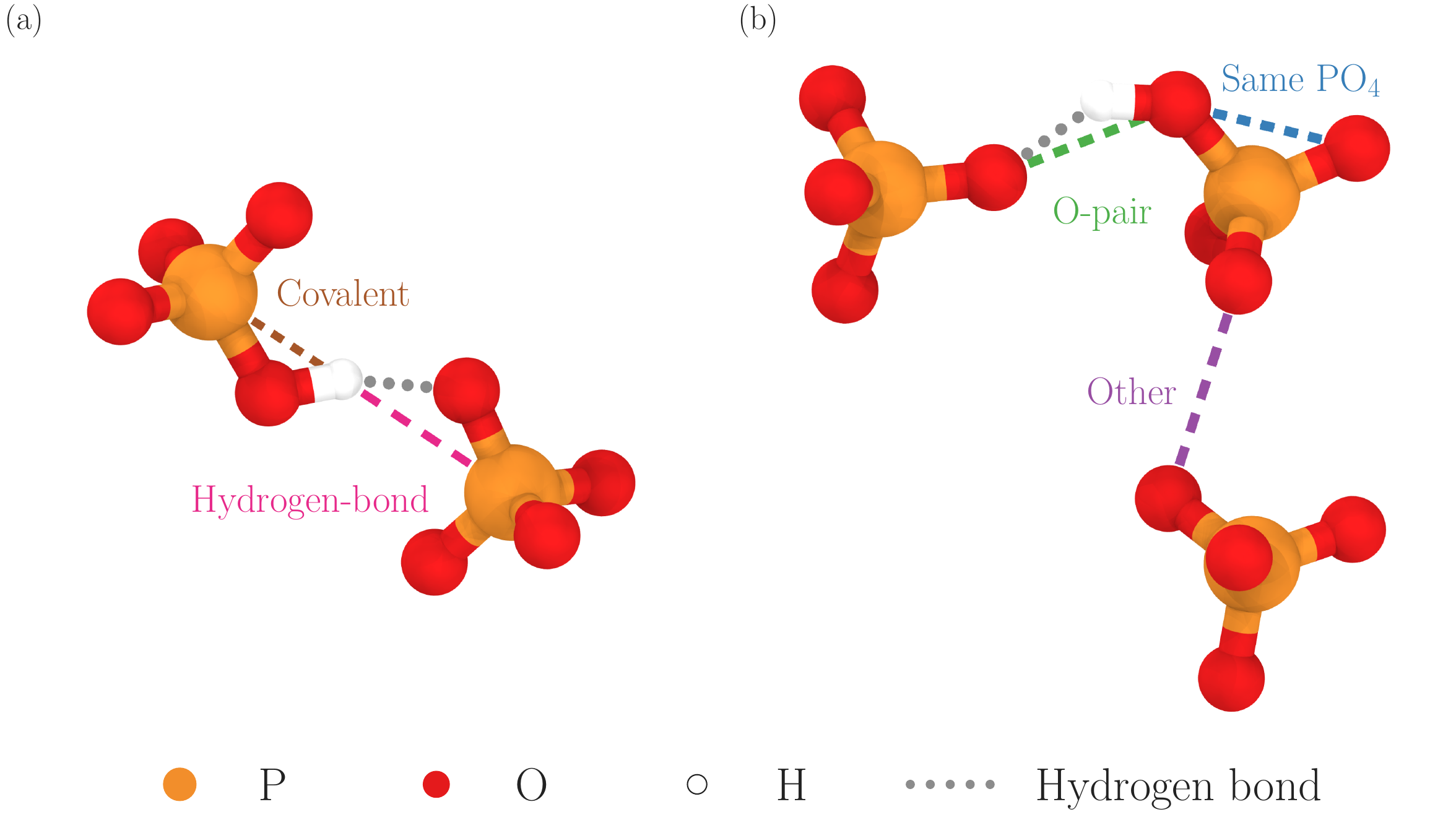}  
  \caption{Edge classifications in CDP.}
  \label{fig:app-CDP}
\end{figure*}

\end{document}